\documentclass[preprint, 10pt, 3p]{elsarticle}

\usepackage{amssymb,amsmath, amsfonts,bm,enumitem,amsthm}
\usepackage{graphicx,epsfig,times,tikz,subfig,ifxetex,soul}
\usepackage{hyperref}

\newtheorem{theorem}{Theorem}[section]
\newtheorem{definition}{Definition}[section]
\newtheorem{lemma}{Lemma}[section]
\newtheorem{corollary}{Corollary}[section]
\newtheorem{proposition}{Proposition}[section]
\newtheorem{example}{Example}[section]
\newtheorem{remark}{Remark}[section]

\journal{Fuzzy Sets and Systems}

\usepackage{xspace}

\newcommand{\Ione}{{\bf(I1)}\xspace}
\newcommand{\Itwo}{{\bf(I2)}\xspace}
\newcommand{\Ithree}{{\bf(I3)}\xspace}

\newcommand{\NP}{{\bf(NP)}\xspace}
\newcommand{\EP}{{\bf(EP)}\xspace}
\newcommand{\IP}{{\bf(IP)}\xspace}
\newcommand{\OP}{{\bf(OP)}\xspace}

\newcommand{\CPN}{{\bf(CP(N))}\xspace}
\newcommand{\LCPN}{{\bf(L$\mhyphen$CP(N))}\xspace}
\newcommand{\RCPN}{{\bf(R$\mhyphen$CP(N))}\xspace}

\usepackage{cellspace} 
\usepackage{adjustbox}
\usepackage{multicol}
\usepackage{multirow}
\usepackage{array}
\newcolumntype{M}[1]{>{\centering\arraybackslash}m{#1}}
\usepackage{pifont}
\def\mystrut(#1,#2){\vrule height #1 depth #2 width 0pt}
\newcolumntype{C}[1]{%
	>{\mystrut(3ex,2ex)\centering}%
	p{#1}%
	<{}}

\mathchardef\mhyphen="2D

\usepackage[normalem]{ulem}
\makeatletter
\renewcommand*\l@figure{\@dottedtocline{1}{1em}{3em}}
\renewcommand*\l@table{\@dottedtocline{1}{1em}{3em}}

\begin{document}

\begin{frontmatter}
\title{On the Non-Uniqueness of Representation of $(U,N)$-Implications}

\author[sas]{Raquel Fernandez-Peralta\corref{cor1}}
\ead{raquel.fernandez@mat.savba.sk}
\author[sas,ce]{Andrea Mesiarov\' a-Zem\' ankov\' a}
\ead{zemankova@mat.savba.sk}
\address[sas]{Mathematical Institute, Slovak Academy of Sciences, Štefánikova 49, 814 73 Bratislava,
Slovakia}
\address[ce]{CE IT4Innovations - Institute for Research and Applications of Fuzzy Modeling,  University of
Ostrava, 30. dubna 22, 701 03 Ostrava, Czech Republic}
\cortext[cor1]{Corresponding author.}

\date{\today}

\begin{abstract}
    Fuzzy implication functions constitute fundamental operators in fuzzy logic systems, extending classical conditionals to manage uncertainty in logical inference. Among the extensive families of these operators, generalizations of the classical material implication have received considerable theoretical attention, particularly $(S,N)$-implications constructed from t-conorms and fuzzy negations, and their further generalizations to $(U,N)$-implications using disjunctive uninorms. Prior work has established characterization theorems for these families under the assumption that the fuzzy negation $N$ is continuous, ensuring uniqueness of representation. In this paper, we disprove  this last fact for $(U,N)$-implications and we show that they do not necessarily possess a unique representation, even if the fuzzy negation is continuous. Further, we provide a comprehensive study of uniqueness conditions for both uninorms with continuous and non-continuous underlying functions. Our results offer important theoretical insights into the structural properties of these operators.
\end{abstract}


\end{frontmatter}

\section{Introduction}\label{section:introduction}

Fuzzy implication functions serve as fundamental operators in fuzzy logic, extending classical logical conditionals to accommodate the continuous nature of approximate reasoning. These functions have become pivotal in computational intelligence due to their ability to model uncertainty-driven inference \cite{Baczynski2008}, finding applications across diverse domains including image processing \cite{DeBaets1998}, knowledge discovery \cite{Fernandez-Peralta2025,Fernandez-Peralta2023}, and intelligent control systems \cite{Mendel2023}. Beyond their practical utility, the theoretical study of these operators is an important research area encompassing the development of new families of fuzzy implication functions, the analysis of their additional properties, and the pursuit of structural representations and characterizations. Indeed, numerous families of these operators have been introduced in the literature under different motivations, along with various additional properties \cite{Fernandez-Peralta2025B}.  This proliferation of operators has underscored the critical importance of studying their characterization, examining their intersections, and revisiting long-standing open problems in the field \cite{Massanet2024}.

Among the vast array of fuzzy implications functions, generalizations of the classical material implication $p \to q \equiv \neg q \vee p$ represent one of the most fundamental and thoroughly investigated classes. Historically, these generalizations first emerged in the form of $S$-implications, constructed using t-conorms and strong negations \cite{Trillas1985}. Subsequent developments expanded this framework to $(S,N)$-implications by incorporating arbitrary fuzzy negations \cite{Klir1995}, and further generalized to $(A,N)$-implications where aggregation functions replace t-conorms \cite{Ouyang2012}. A particularly notable subclass arises when disjunctive uninorms are used, leading to $(U,N)$-implications \cite{DeBaets1999}, which is a family that is suitable to define a fuzzy morphology based on uninorms \cite{Gonzalez-Hidalgo2015}.

Throughout the years, there have been several advances in the characterization of these families of fuzzy implication functions: In \cite{Baczynski2007}, the authors characterized $(S,N)$-implications with continuous negations, while \cite{Backzynski2009,Libo2017,Zhou2020} extended these results to uninorms, 2-uninorms and grouping functions under the same continuity condition. Later, \cite{Fernandez-Peralta2022} provided a representation theorem of $(S,N)$-implications for non-continuous negations, linking the problem to the completion of t-conorms in undefined regions \cite{Fernandez-Peralta2023B,Fernandez-Peralta2023C}. A key step in these results is the reconstruction of the aggregation function from the given negation and implication. Depending on the conditions, this reconstruction may or may not be unique—a direct consequence of whether the fuzzy implication's representation in terms of the aggregation function is unique. For example, in the case of $(S,N)$-implications, uniqueness holds when the negation $N$ is continuous but fails otherwise. Similarly, \cite{Backzynski2009} claimed that $(U,N)$-implications have a unique representation when the negation $N$ is continuous. However, in this work we prove that $(U,N)$-implications do not generally possess a unique representation, even when $N$ is continuous.

This finding raises new questions regarding the conditions that guarantee uniqueness of representation, as well as the characterization of all uninorms generating the same fuzzy implication function. In this paper, we provide a comprehensive investigation of both problems, considering uninorms with both continuous and non-continuous underlying functions \cite{Mesiarova2017OS,Mesiarova2018}.

The structure of the paper is as follows. First, in Section \ref{section:preliminaries} we include the necessary background. In Section \ref{section:corrigendum} we prove that the representation of $(U,N)$-implications is not unique. In Section \ref{section:uniqueness} we present the results regarding the uniqueness of representation. The paper ends in Section \ref{section:conclusions} with some conclusions and future work.

\section{Preliminaries}\label{section:preliminaries}


In this section, we provide the results and definitions necessary for making this paper self-contained. For more information regarding fuzzy implications and aggregation functions we refer the reader to the books \cite{Baczynski2008,Beliakov2010,Calvo2002,Grabisch2009,Klement2000}.

A fuzzy negation $N$ is a decreasing unary function which is the generalization of the classical negation.
\begin{definition}[\bf{\cite{Fodor1994}}]\label{def:fuzzy_negation}
	A unary function $N:[0,1] \to [0,1]$ is said to be a \emph{fuzzy negation} if it satisfies:
	\begin{description}
		\item[(N1)] $N(0)=1$ and $N(1)=0$. \hfill (Boundary Condition)
		\item[(N2)] $N$ is decreasing. \hfill (Monotonicity)
	\end{description}
A fuzzy negation $N$ is called
\begin{enumerate}[label=(\roman*)]
\item \emph{strict}, if it is strictly decreasing and continuous.
\item \emph{strong}, if it is an involution, i.e., $N(N(x))=x$ for all $x \in [0,1]$.
\end{enumerate}
\end{definition}
It is straightforward to prove that any strong fuzzy negation is strict.

For some results, it is useful to consider the pseudo-inverse of fuzzy negations which is defined as follows.
\begin{definition}[\bf{\cite{Baczynski2007}}]\label{def:pseudo-negation}
	Let $N$ be a fuzzy negation, the \emph{modified pseudo-inverse} of $N$ is denoted by $\mathfrak{R}_N$ and it is defined as follows
	\begin{equation}\label{eq:modified_pseudoinverse_negation}
		\mathfrak{R}_N(x)
		=
		\left\{ \begin{array}{ll}
			1 &   \text{if }   x=0, \\
			N^{(-1)}(x) & \text{if } x \in ]0,1], \\
		\end{array} \right.
	\end{equation}
	where $N^{(-1)}$  is the pseudo-inverse of $N$ given by
	$$N^{(-1)}(x)=\sup \{y \in [0,1] \, | \, N(y)>x\}, \quad \text{for all } x \in [0,1].$$
\end{definition}
The pseudo-inverse of a continuous negation is a strictly decreasing fuzzy negation that satisfies the inverse function equalities, but one of them only when restricted to its range.

\begin{proposition}[\bf{\cite[Proposition 3.13]{Baczynski2007}}]\label{prop:RN} 
If $N$ is a continuous fuzzy negation, then $\mathfrak{R}_N$ is a strictly decreasing fuzzy negation. Moreover, the following statements hold:
\begin{enumerate}[label=(\roman*)]
\item $\mathfrak{R}_N^{(-1)} = N$.
\item $N \circ \mathfrak{R}_N = \text{id}_{[0,1]}$.
\item $\mathfrak{R}_N \circ N \mid_{\text{Ran }\mathfrak{R}_N} = \text{id}_{\text{Ran }\mathfrak{R}_N}$.
\end{enumerate}
\end{proposition}

Bivariate aggregation functions were originally developed as generalizations of classical conjunctions and disjunctions.
\begin{definition}[\bf{\cite[Definition 3]{Pradera2016}}] 
A \emph{bivariate aggregation function} is an increasing function $A:[0,1]^2 \to [0,1]$ verifying the boundary conditions $A(0,0)=0$ and $A(1,1)=1$.
\end{definition}
Since in this work we only deal with two variable aggregation functions, we omit the term bivariate for simplification. An aggregation function is called \emph{conjuctive} if $A \leq \min$ and \emph{disjunctive} if $A \geq \max$.

Among the most important aggregation functions are uninorms, characterized by associativity and the existence of a neutral element $e \in [0,1]$.
\begin{definition}[\bf{\cite{Klement2000,Yager1996}}]\label{def:uninorm}
	A binary operator $U:[0,1]^2 \to [0,1]$ is said to be a \emph{uninorm} if there exists $e \in [0,1]$, called \emph{neutral element}, such that $U$ satisfies:
	\begin{description}
		\item[(U1)] $U(x,y)=U(y,x)$ for all $x,y \in [0,1]$. \hfill (Commutativity)
		\item[(U2)] $U(x,U(y,z))=U(U(x,y),z)$ for all $x,y,z \in [0,1]$. \hfill (Associativity)
		\item[(U3)] $U(x,y) \leq U(x,z)$ when $y \leq z$, for all $x,y,z \in [0,1]$. \hfill (Monotonicity)
		\item[(U4)] $U(x,e)=x$ for all $x \in [0,1]$. \hfill (Neutral element)
	\end{description}
\end{definition}
For all uninorms, it holds that $U(1,0) \in \{0,1\}$ and we say that $U$ is \emph{conjunctive} if $U(1,0)=0$ and \emph{disjunctive} if $U(1,0)=1$. Besides, if in Definition \ref{def:uninorm} the neutral element is $e=1$, we recover the most common conjunctive aggregation functions, namely \textit{t-norms}, whereas setting $e=0$ yields the best-known disjunctive operations, which are \textit{t-conorms}. In order to distinguish from these particular cases, we say that a uninorm is \emph{proper} if $e \in ]0,1[$.
\begin{proposition}[\bf{\cite[Proposition 1]{Mas2015}}] Let $U$ be a uninorm with neutral element $e \in ]0,1[$. Then $U$ has the following representation
\begin{equation}\label{eq:uninorm}
    U(x,y)
    =
    \left\{\begin{array}{ll}
        e \cdot T_U\left(\frac{x}{e},\frac{y}{e}\right) & \text{if } x,y \in [0,e], \\
        e + (1-e) \cdot S_U\left(\frac{x-e}{1-e},\frac{y-e}{1-e}\right) & \text{if } x,y \in [e,1], \\
        C(x,y) & \text{otherwise,}
    \end{array}
    \right.
\end{equation}
where $T_U$ is a t-norm, $S_U$ is a t-conorm and $C:[0,e] \times [e,1] \cup [e,1] \times [0,e] \to [0,1]$ fulfills $\min \{x,y\} \leq C(x,y) \leq \max\{x,y\}$ for all $(x,y) \in [0,e] \times [e,1] \cup [e,1] \times [0,e]$. In this case, $T_U$ and $S_U$ are called the underlying t-norm and t-conorm, respectively.
\end{proposition}
A uninorm has continuous underlying functions if its underlying t-norm and underlying t-conorm are both continuous. The class of uninorms with continuous underlying functions is denoted as \textbf{COU}.

The structure of uninorms is more complex than t-norms and t-conorms, so usually these operators are classified in many different families \cite{Mas2015}. Some well-known characterized families of uninorms can be found hereunder.
\begin{theorem}[\textbf{\cite[Theorem 1]{Mas2015}}] Let $U$ be a uninorm with neutral element $e \in ]0,1[$.
    \begin{enumerate}[label=(\roman*)]
    \item If $U(0,1)=0$ and the function $U(\cdot,1)$ is continuous except in $x=e$, then $U$ is given by Eq. (\ref{eq:uninorm}) with $C=\min$. This class of uninorms is denoted by $\mathcal{U}_{\min}$.
    \item If $U(0,1)=1$ and the function $U(\cdot,0)$ is continuous except in $x=e$, then $U$ is given by Eq. (\ref{eq:uninorm}) with $C=\max$. This class of uninorms is denoted by $\mathcal{U}_{\max}$.
    \end{enumerate}
\end{theorem}

\begin{definition}[\textbf{\cite[Definition 5.1.5]{Baczynski2008}}]
	A uninorm $U$ such that $U(x,x)=x$ for all $x \in [0,1]$ is said to be an \emph{idempotent uninorm}. The class of all idempotent uninorms is denoted by $\mathcal{U}_{Idem}$.
\end{definition}
The characterization of idempotent uninorms can be found in \cite{Mas2015,Ruiz-Aguilera2010} and it is based on a decreasing function $g : [0,1] \to [0,1]$ which fulfills specific conditions.


\begin{definition}[\textbf{\cite[Theorem 5.1.12]{Baczynski2008}}]
	A uninorm $U$ is called \emph{representable} if it has a continuous additive generator, i.e., there exists a continuous and strictly increasing function $h:[0,1] \to [-\infty,+\infty]$ such that $h(0)=-\infty$, $h(e)=0$ for an $e \in ]0,1[$ and $h(1)=+\infty$, which is uniquely determined up to a positive multiplicative constant, such that
	$$U(x,y)
	=
	\left\{\begin{array}{ll}
		0 & \text{if } (x,y) \in \{(0,1),(1,0)\}, \\
		h^{-1}(h(x)+h(y)) & \text{otherwise},
	\end{array}
	\right.
	$$
	or 
	$$U(x,y)
	=
	\left\{\begin{array}{ll}
		1 & \text{if } (x,y) \in \{(0,1),(1,0)\}, \\
		h^{-1}(h(x)+h(y)) & \text{otherwise}.
	\end{array}
	\right.
	$$
	The class of all representable uninorms is denoted by  $\mathcal{U}_{Rep}$.
\end{definition}

Now, regarding the classical conditional, the generalization to fuzzy logic is given by fuzzy implication functions which are defined as follows.

\begin{definition}[\bf{\cite{Baczynski2008,Fodor1994}}]\label{defimp}
	A binary operator $I:[0,1]^2 \to [0,1]$ is said to be a \emph{fuzzy implication function} if it satisfies:
	\begin{description}
		\item[\Ione]  $I(x,z)\geq I(y,z)\ $  when  $\ x\leq y$, for all $x,y,z\in[0,1]$. \hfill (Left Antitonicity)
		\item[\Itwo]  $I(x,y)\leq I(x,z)\ $  when  $\ y\leq z$, for all $x,y,z\in[0,1]$. \hfill (Right Isotonicity)
		\item[\Ithree]  $I(0,0)=I(1,1)=1$ and $I(1,0)=0$. \hfill (Boundary Conditions)
	\end{description}
\end{definition}
From Definition \ref{defimp} it can be easily derived that $I(0,x) = I(x,1) = 1$ for all $x \in [0, 1]$. However, the values $I(x,0)$ and $I(1,x)$ are not predetermined by the definition. In fact, the values $I(x,0)$ define what is called the natural negation of a fuzzy implication function.
\begin{definition}[\textbf{\cite[Definition 1.4.14]{Baczynski2008}}]\label{def:naturalnegationI}
	Let $I$ be a fuzzy implication function. The function $N_I:[0,1] \to [0,1]$ defined by
	$$N_I(x)=I(x,0), \quad \text{for all } x \in [0,1],$$
	is a fuzzy negation called the \emph{natural negation} of $I$.
\end{definition}
Apart from the natural negation which considers the horizontal section of the fuzzy implication function at $y=0$, for those operators related to uninorms is of use to consider other horizontal cuts that may be fuzzy negations.

\begin{definition}[\textbf{\cite[Definition 3.4]{Backzynski2009}}] Let $I:[0,1]^2 \to [0,1]$ be any function and $\alpha \in [0,1[$. If the function $N_I^{\alpha}:[0,1] \to [0,1]$ given by
$$N_I^{\alpha} = I(x,\alpha), \quad \text{for all } x \in [0,1],$$
is a fuzzy negation, then it is called the \emph{natural negation of $I$ with respect to $\alpha$}.
\end{definition}

Beyond the basic definition, additional properties are often studied and imposed depending on the context. Among the numerous possible properties (see \cite[Table 1]{Fernandez-Peralta2025} for a comprehensive list), we focus here on those most relevant to our work: 

\begin{itemize}[noitemsep,topsep=0pt,leftmargin=*,labelwidth=4.5em,align=left]
    \item[\text{\NP}] \emph{Left neutrality}: \(I(1,y) = y\), \(y \in [0,1]\).
    \item[\text{\EP}] \emph{Exchange principle}: \(I(x,I(y,z)) = I(y,I(x,z))\), \(x,y,z \in [0,1]\).
    \item[\text{\IP}] \emph{Identity principle}: \(I(x,x) = 1\), \(x \in [0,1]\).
    \item[\text{\OP}] \emph{Ordering property}: \(I(x,y)=1 \Leftrightarrow x \leq y\), \(x,y \in [0,1]\).
    \item[\text{\CPN}] \emph{Contrapositive symmetry} (w.r.t. \(N\)): \(I(x,y) = I(N(y),N(x))\), \(x,y \in [0,1]\).
    \item[\text{\LCPN}] \emph{Left contraposition} (w.r.t. \(N\)): \(I(N(x),y) = I(N(y),x)\), \(x,y \in [0,1]\).
    \item[\text{\RCPN}] \emph{Right contraposition} (w.r.t. \(N\)): \(I(x,N(y)) = I(y,N(x))\), \(x,y \in [0,1]\).
\end{itemize}

Furthermore, numerous families of fuzzy implication functions have been studied in the literature (see \cite{Fernandez-Peralta2025}). In this work, we focus specifically on generalizations of the material implication known as $(A,N)$-implications, where $A$ represents a given aggregation function.

\begin{definition}\label{def:(A,N)implications} 
A function $I:[0,1]^2 \to [0,1]$ is called an \emph{$(A,N)$-implication} if there exist an aggregation function $A$ and a fuzzy negation $N$ such that
\begin{equation}\label{eq:(A,N)-implication}
I(x,y)=A(N(x),y), \quad x,y \in [0,1].
\end{equation}
If $I$ is an $(A,N)$-implication generated from $A$ and $N$, then it is denoted by $I_{A,N}$.
\end{definition}
The function defined in Eq. (\ref{eq:(A,N)-implication}) satisfies the conditions for being a fuzzy implication function if and only if $A$ is a disjunctor (see \cite[Thorem 33]{Pradera2016}). However, since the structure of this family highly depends on the underlying aggregation function, it is typically divided into distinct subfamilies. Below we highlight the three most relevant cases, for a more complete list we refer the reader to \cite[Table 4]{Fernandez-Peralta2025}.
\begin{itemize}
\item If $A=S$ is a t-conorm and $N$ is a strong negation, then $I_{A,N}$ belongs to the family of $S$-implications \cite{Trillas1985}.
\item If $A=S$ is a t-conorm, then $I_{A,N}$ belongs to the family of $(S,N)$-implications \cite{Baczynski2008}.
\item If $A=U$ is a uninorm then $I_{A,N}$ belongs to the family of $(U,N)$-implications \cite{DeBaets1999}.
\end{itemize}
Since this family of fuzzy implication functions is one of the most important ones, a problem of interest throughout the last decades has been to provide a characterization in terms of their properties. Indeed, in \cite[Theorem 3.2]{Trillas1985} the authors already provided a characterization for $S$-implications based on the properties \NP, \EP and \CPN. Later, \cite[Theorem 5.2]{Baczynski2007} extended this result to continuous negations, showing that only \Ione and \EP are required in this case. For the non-continuous case, it was proved in \cite[Theorem 28]{Fernandez-Peralta2022C} that the characterization problem is related to the completion of the corresponding t-conorm unknown in certain regions that depend on the discontinuity points of the fuzzy negation. This problem has already been solved in the case when $N$ has one point of discontinuity \cite{Fernandez-Peralta2023B,Fernandez-Peralta2023C}. On the other hand, for continuous negations and uninorms the following characterization was obtained in \cite{Backzynski2009}.

\begin{theorem}[\textbf{\cite[Theorem 6.4]{Backzynski2009}}]\label{th:charac_un} 
For a function $I:[0,1]^2 \to [0,1]$ the following statements are equivalent:
\begin{enumerate}[label=(\roman*)]
\item $I$ is an $(U,N)$-implication generated from some uninorm $U$ with the neutral element $e \in ]0,1[$ and some continuous fuzzy negation $N$.
\item $I$ satisfies \Ione, \Ithree, \EP and $N_I^e$ is a continuous negation for some $e \in ]0,1[$.
\end{enumerate}
Moreover, the representation of the $(U,N)$-implication is unique with
$$N(x) = N_I^e(x), \quad x \in [0,1], \quad U(x,y) = I(\mathfrak{R}_N(x),y), \quad x,y \in [0,1].$$
\end{theorem}
Differently from $(S,N)$-implications, in this case the boundary conditions \Ithree need to be imposed and the representation is given in terms of natural negation with respect to $e$. Moreover, in Theorem \ref{th:charac_un} the authors claim that, similarly to $(S,N)$-implications, the representation of the $(U,N)$-implication is unique. Nonetheless, in the upcoming sections we prove that this statement does not hold and we study the conditions under which uniqueness of representation is guaranteed.

\section{Corrigendum: non-uniqueness of the representation of $(U,N)$-implications even if $N$ is continuous}\label{section:corrigendum}

In this section, we recall the part of the proof of Theorem \ref{th:charac_un} that claims unicity of representation, we highlight in which step the reasoning of the authors does not hold and we then provide some general results that highlight when we can ensure that the representation of a $(U,N)$-implication is unique.

Let $I$ be a fuzzy implication functions and let us assume that there exist two continuous fuzzy negations $N_1$, $N_2$ and two uninorms $U_1$, $U_2$ with neutral elements $e_1, e_2 \in ]0,1[$, respectively, such that $I(x,y) = U_1(N_1(x),y) = U_2(N_2(x),y)$ for all $x,y \in [0,1]$. According to \cite{Backzynski2009}, in this case it must happen that $N_1 = N_2 = N_I^{e_1} = N_I^{e_2}$ and from there and with the assumption of the continuity of the negation they prove that $U_1=U_2$. Nonetheless, in the following example we show that this is not necessary the case. In particular, we later show that the equality $N_1 = N_2 = N_I^{e_1} = N_I^{e_2}$ holds only if $e_1=e_2$.



\begin{example}\label{example:1} Let $U_1,U_2:[0,1]^2 \to [0,1]$ and $N_1,N_2:[0,1] \to [0,1]$ be functions given by
$$N_1(x) = 1-x, \quad N_2(x) = \frac{1-x}{1+2x},$$
    $$U_1(x,y) = \begin{cases}
    \frac{xy}{xy+(1-x)(1-y)}   &\text{ if $(x,y)^2\in [0,1]^2\setminus \{(0,1),(1,0)\},$} \\
    1 &\text{otherwise,}\end{cases}$$
    $$U_2(x,y) = \begin{cases} \frac{3xy}{3xy+(1-x)(1-y)}  &\text{ if $(x,y)^2\in [0,1]^2\setminus \{(0,1),(1,0)\},$} \\ 1 &\text{otherwise.}\end{cases} $$
        Then $U_1$ is a disjunctive uninorm  with neutral element $e_1=\frac{1}{2},$ $U_2$ is a disjunctive uninorm  with neutral element
    $e_2=\frac{1}{4},$
    and $N_1$ and $N_2$ are continuous fuzzy negations.
Moreover, for $(x,y)^2\in [0,1]^2\setminus \{(0,0),(1,1)\},$ we have
$$U_1(N_1(x),y) = U_2(N_2(x),y) = \frac{(1-x)y}{(1-x)y+x(1-y)},$$
$U_1(N_1(1),1) = U_2(N_2(1),1) = 1$ and $U_1(N_1(0),0) = U_2(N_2(0),0) = 1.$
\end{example}

As stated before, the issue in the proof of Theorem \ref{th:charac_un} arises from assuming that $N_1=N_2$ under continuity conditions. In the following proposition, we show that to ensure uniqueness of representation, we must either know that $U_1=U_2$ or that the negations are continuous and $N_1=N_2$ or $e_1=e_2$.

\begin{proposition}\label{proposition:1} 
Let $I:[0,1]^2 \to [0,1]$ be a fuzzy implication function, $U_1,U_2:[0,1]^2 \to [0,1]$ two disjunctive uninorms with neutral elements $e_1,e_2 \in ]0,1[,$ respectively and $N_1,N_2:[0,1] \to [0,1]$ two fuzzy negations such that
$$I(x,y) = U_1(N_1(x),y) = U_2(N_2(x),y).$$
Then, the following statements hold:
\begin{enumerate}[label=(\roman*)]
    \item If $U_1=U_2$ then $N_1=N_2$.
    \item If $N_1,N_2$ are continuous and $N_1=N_2$ then $U_1=U_2$.
    \item If $e_1=e_2$ and $N_1$,$N_2$ are continuous then $U_1=U_2$ and $N_1=N_2$.
\end{enumerate}    
\end{proposition}

\begin{proof}
\begin{enumerate}[label=(\roman*)]
    \item If  $U_1=U_2 = U$ is a disjunctive uninorm with neutral element $e \in ]0,1[$ then for each $x\in [0,1]$
    $$N_1(x)=U(N_1(x),e)=I(x,e)=U(N_2(x),e)=N_2(x).$$
    \item Let $x_0,y_0 \in [0,1]$, since $N$ is a continuous fuzzy negation for any $x_0 \in [0,1]$ there exists $x_1 \in [0,1]$ such that $N(x_1)=x_0$ and then
    $$U_1(x_0,y_0) = U_1(N(x_1),y_0) = I(x_1,y_0) = U_2(N(x_1),y_0) =
    U_2(x_0,y_0).$$
    Thus, $U_1=U_2$.
    \item If $e_1 = e_2 = e$ then for any $x\in [0,1]$
$$N_1(x)=U_1(N_1(x),e)=I(x,e)=U_2(N_2(x),e)=N_2(x),$$
and by Point (ii) we obtain $U_1 = U_2$.
\end{enumerate}
\end{proof}

In the case of considering non-continuous negations then it is clear that, similarly to $(S,N)$-implications, we also may not have a unique representation. Indeed, in this case since the fuzzy implication function is defined via the Eq. $U(N(x),y)$, the horizontal cuts are only evaluated in the range of $N$, which is not $[0,1]$ anymore. Below, we illustrate this with an explicit example.

\begin{example}
Let $U_1$ and $U_2$ be two disjunctive uninorms, with neutral elements $e_1,e_2,$ respectively, such that $U_1(0,x)=0=U_2(0,x)$ for all $x\in [0,1[$ . Let $N_i$ for $i\in \{1,2\}$ be two fuzzy negations given by $$N_i(x) = \begin{cases}
    1&\text{ if $x=0,$} \\
    0&\text{ if $x=1,$}\\
    e_i&\text{ otherwise.}\end{cases}$$ 
Then $U_1(N_1(x),y)=U_2(N_2(x),y)=I(x,y)$ for all $x,y\in [0,1],$ where $$I(x,y) = \begin{cases} 
    1 &\text{ if $x=0$ or $y=1,$}\\
    0 &\text{ if $x=1$ and $y\in [0,1[,$}\\
    y &\text{otherwise.}
    \end{cases}$$
    There are many uninorms that may be considered in this example, such as any disjunctive uninorm in $\mathcal{U}_{rep}$, any uninorm in
    of the form (ii) or (vi) in Theorem 3 from \cite{LiLiu2016} (including those with non-Archimedean continuous underlying functions).
    Note that all these uninorms have positive underlying t-conorms since otherwise $U(x,y)=1$ holds for some $x,y\in [e,1[$ and $$1=U(0,1)=U(0,U(x,y))=U(U(0,x),y)=U(0,y)=0,$$ which is a contradiction.

\end{example}

As a first step, in this paper we focus on investigating the properties of a disjunctive uninorm $U$ such that the representation of any $(U,N)$-implication function is unique  for  continuous fuzzy negations.

\section{On the unique representation of $(U,N)$-implications with a continuous negation}\label{section:uniqueness}

In this section, we first establish the conditions under which the representation of $(U,N)$-implications is unique under the continuity assumption. Subsequently, we examine uninorms that fail to satisfy these conditions.

\subsection{Uniqueness of representation}
We start by proving that if a $(U,N)$-implication generated by a continuous fuzzy negation does not have a unique representation, i.e., it is generated by two uninorms and two fuzzy negations, where either uninorms or negations
 differ from each other, then both uninorms and negations have to be necessarily different and, moreover, the uninorms have a continuous horizontal cut in  a point different from the corresponding neutral element.

\begin{proposition}\label{prop:non_uniq}
    Let $N_1, N_2:[0,1] \to [0,1]$  be two continuous fuzzy negations, $U_1, U_2:[0,1]^2 \to [0,1]$ two disjunctive uninorms with neutral elements $e_1,e_2 \in ]0,1[,$ respectively,
     and $I_{U_1,N_1}, I_{U_2,N_2}$ the corresponding $(U,N)$-implications.
     If $I_{U_1,N_1} = I_{U_2,N_2}$ and either $e_1\neq e_2,$ or $N_1\neq N_2,$ or $U_1\neq U_2$ then 
\begin{enumerate}[label=(\roman*)]
     \item $e_1\neq e_2$ and $N_1\neq N_2,$ and $U_1\neq U_2,$
     \item $f_1\colon [0,1]\longrightarrow [0,1]$ given by $f_1(x)=U_1(x,e_2)$ for $x\in [0,1],$ is a continuous, increasing function such that $f_1(0)=0$ and $f_1(1)=1,$ 
     \item  $f_2\colon [0,1]\longrightarrow [0,1]$ given by $f_2(x)=U_2(x,e_1)$ for $x\in [0,1],$ is a continuous, increasing function such that $f_2(0)=0$ and $f_2(1)=1$. 
 \end{enumerate}
\end{proposition}

\begin{proof} 
Point (i) follows from Proposition
\ref{proposition:1}. 
Since function $f_1$ is defined as
$f_1(x) = U_1(x,e_2)$ for all $x \in [0,1],$ it is clear that $f_1$ is increasing and since $I_{U_1,N_1} = I_{U_2,N_2}$ we have
$$
f_1(0) = U_1(0,e_2) = U_1(N_1(1),e_2) = I_{U_1,N_1}(1,e_2) = I_{U_2,N_2}(1,e_2) = U_2(N_2(1),e_2) = N_2(1)=0,
$$
$$
f_1(1) = U_1(1,e_2) = U_1(N_1(0),e_2) = I_{U_1,N_1}(0,e_2) = I_{U_2,N_2}(0,e_2) = U_2(N_2(0),e_2) = N_2(0)=1.
$$
Next, let us assume that $f_1$ is not continuous, i.e., there exists $a \in [0,1]$ such that $\lim_{x \to a} f_1(x) \not = f_1(a) = U_1(a,e_2)$. If $\lim_{x \to a^{-}} f_1(x) \not = f_1(a)$ then there exists $r \in [0,1]$ such that
$$
U_1(s,e_2) < r < U_1(a,e_2) \quad \text{for all } s < a.
$$
Denote $\tilde{x} = \sup\{x\in [0,1]\mid N_1(x)=a\}.$
Since $N_1$ is a continuous decreasing function then $N_1(\tilde{x})=a$ and  $N_1(x) < a$ for all $x > \tilde{x}$. Moreover,
$$
U_1(N_1(\tilde{x}),e_2) = U_1(a,e_2),
$$
$$
U_1(N_1(x),e_2) <r, \quad \text{for all } x > \tilde{x}.
$$
Since $$
U_1(N_1(x),e_2) = U_2(N_2(x),e_2) = N_2(x), \quad \text{for all } x \in [0,1],
$$
 we obtain
$$
\lim_{x \to \tilde{x}^{+}} N_2(x)=\lim_{x \to \tilde{x}^{+}} U_1(N_1(x),e_2) \leq r < U_1(N_1(\tilde{x}),e_2) =N_2(\tilde{x}),
$$
which is a contradiction with the fact that $N_2$ is a continuous function.

On the other hand, if $\lim_{x \to a^{+}} f_1(x) \not = f_1(a)$ we arrive to contradiction in an analogous manner. Thus, $f_1$ is a continuous function. Finally, Point (iii) can be proven similarly to Point (ii). 
\end{proof}

Thanks to the proposition above, we can directly affirm that the representation of a $(U,N)$-implication with a fuzzy continuous negation $N$ and a uninorm $U$ is not unique if and only if $U$ has a continuous horizontal cut, in a point different from the neutral element of $U$, whose range is $[0,1]$.

\begin{proposition}\label{proposition:2}
    Let $N:[0,1] \to [0,1]$ be a continuous fuzzy negation, $U$ a disjunctive uninorm with neutral element $e \in ]0,1[$
     and $I$ the corresponding $(U,N)$-implication. Then, the representation of $I$ as a $(U,N)$-implication with a continuous fuzzy negation is not unique if and only if there exists an $\alpha \in ]0,1[$, $\alpha \not = e$ such that $f(x)=U(\cdot,\alpha)$ is a continuous, increasing function with $f(0)=0$ and $f(1)=1$.
\end{proposition}

\begin{proof}
    \begin{itemize}
    \item[$(\Rightarrow)$] Directly from Proposition \ref{prop:non_uniq}.
    \item[$(\Leftarrow)$] By hypothesis, function $N_I^\alpha(x) = U(N(x),\alpha)$ is a continuous fuzzy negation. Moreover, 
    by   \cite[Proposition 5.2]{Backzynski2009}  $(U,N)$-implication $I$ satisfies (EP).
        Now, let us define
    $$U_{I,\alpha}(x,y) = I(\mathfrak{R}_{N_I^{\alpha}}(x),y) \quad \text{for all } x,y \in [0,1].$$
    We know by \cite[Corollary 6.3]{Backzynski2009} that $U_{I,\alpha}$ is a disjunctive uninorm with neutral element $\alpha$. Now, to prove that $U_{I,\alpha}(N_I^\alpha(x),y) = I(x,y)$ for all $x,y\in [0,1]$ we need to distinguish between two cases:
    \begin{itemize}
        \item If $x \in \text{Ran } \mathfrak{R}_{N_I^{\alpha}}$ then by (iii)-Proposition \ref{prop:RN} we have that
        $$U_{I,\alpha}(N_I^\alpha(x),y) = I(\mathfrak{R}_{N_I^{\alpha}} \circ N_I^\alpha (x),y) = I(x,y).$$
        \item  If $x \not \in \text{Ran } \mathfrak{R}_{N_I^{\alpha}}$, from the continuity of $N_I^\alpha$ we know that there exists $\tilde{x} \in \text{Ran } \mathfrak{R}_{N_I^{\alpha}}$ such that $N_I^\alpha(x) = N_I^\alpha(\tilde{x})$. For $y \in [0,1]$, again by the continuity of $N_I^\alpha$ there exists $\tilde{y} \in [0,1]$ such that $N_I^\alpha(\tilde{y}) = y$ and since $I$ satisfies {\bf(R$\mhyphen$CP($N_I^{\alpha}$))}  (see \cite[Lemma 3.5]{Backzynski2009}) we have
        $$
        I(x,y) = I(x,N_I^\alpha(\tilde{y}))
        = I(\tilde{y},N_I^\alpha(x))=I(\tilde{y},N_I^\alpha(\tilde{x})) = I(\tilde{x},N_I^\alpha(\tilde{y})) = I(\tilde{x},y).
        $$
        Thus, 
        $$U_{I,\alpha}(N_I^\alpha(x),y) = U_{I,\alpha}(N_I^\alpha(\tilde{x}),y) = I(\tilde{x},y)=I(x,y).$$
    \end{itemize}
    \end{itemize}
\end{proof}

By the proof of Proposition \ref{proposition:2} we can deduce that horizontal cuts totally determine the representations of $(U,N)$-implications with continuous negations.

\begin{corollary}\label{cor:representations} 
Let $I$ be a $(U,N)$-implication generated from some uninorm $U$ with neutral element $e \in ]0,1[$ and some continuous fuzzy negation $N$. Then, any representation of $I$ as a $(U,N)$-implication with a uninorm $U^*$ and a continuous fuzzy negation $N^*$ is given by
$$
N^*(x) = N_I^{\alpha}(x), \quad x \in [0,1], \quad U^*(x,y) = I (\mathfrak{R}_N(x),y), \quad x,y \in [0,1],
$$
where $\alpha \in (0,1)$ and the horizontal cut $I(\cdot,\alpha)$ is a fuzzy continuous negation. In particular, at least $\alpha=e$ fulfills these conditions and the representation is unique if and only if $I(\cdot,e)$ is the only horizontal cut which is a continuous fuzzy negation.
\end{corollary}

Corollary \ref{cor:representations} jointly with the first part of Theorem \ref{th:charac_un} characterizes $(U,N)$-implications with a continuous fuzzy negation and their representations. Moreover, it shows that only fuzzy implication functions with a unique horizontal cut which is a fuzzy negation have uniqueness of representation. If we assume that this fuzzy implication function is generated by a continuous fuzzy negation $N$ and a disjunctive uninorm $U$, then the equivalent condition is that the only continuous, increasing horizontal cut with $U(0,\alpha)=0$ and $U(1,\alpha)=1$ is the one corresponding to the neutral element of the uninorm $\alpha=e$.
\begin{example}\label{example:unique_rep}
Let us consider $e \in ]0,1[$, $N$ a continuous fuzzy negation and $U$ a disjunctive uninorm given by
$$U(x,y)=\begin{cases}
1 &\text{if $\max(x,y)=1,$} \\
\max(x,y) &\text{if $x,y\in ]e,1[,$}\\
0 &\text{if $x,y\in [0,e[,$}\\
x &\text{if $y=e,$}\\
y &\text{if $x=e,$}\\
\min(x,y) &\text{otherwise,}\\
\end{cases}$$
i.e., the underlying functions of $U$ are
 the maximum  t-conorm and the drastic product t-norm. In this case, the horizontal cuts are the following
 $$
U(x,\alpha)=1  \quad
    \text{for $\alpha =1$,}$$
$$
U(x,\alpha) = \begin{cases} 
    x &\text{ if $x < e$,}\\
    \alpha &\text{ if $e\leq x\leq \alpha$,}\\
     x &\text{ if $x > \alpha$,}\\
    \end{cases}
    \quad
    \text{for all $\alpha > e$,}
$$
$$
U(x,\alpha) = \begin{cases} 
    0 &\text{ if $x < e$,}\\
    \alpha &\text{ if $e \leq x < 1$,}\\
    1 &\text{ if $x=1$,}\\
    \end{cases}
    \quad
    \text{for all $\alpha<e$.}
$$
Since not any horizontal cut except for $y = e$ is continuous with range $[0,1]$, the fuzzy implication function $I(x,y)=U(N(x),y)$ has a unique representation as a $(U,N)$-implication with a continuous fuzzy negation.
\end{example}

In the following subsections, we will denote by $\mathcal{U}_x^{hc}$
the set of all disjunctive uninorms defined on $[0,1],$ 
which have a continuous horizontal cut, with range $[0,1],$ in point $x\in [0,1]$ different from neutral element and we 
will study uninorms from  $\mathcal{U}_x^{hc}$. We will first consider uninorms with continuous underlying functions, which were characterized in \cite{Mesiarova2017OS} and then we will focus on uninorms with non-continuous underlying functions, whose complex structure does not allow for complete characterization.


\subsection{Uninorms with continuous underlying functions}\label{subsec:cont_und}

In this section, apart from the continuity of the fuzzy negation we also assume that the involved uninorm has continuous underlying functions. In this case, the representation of the corresponding $(U,N)$-implication is not unique and it is determined by any uninorm which is a linear transformation of some representable uninorm inside a sub-square of $[0,1]^2$ which is delimited by idempotent points.

\begin{proposition} \label{proposition:3}
   Let $N_1, N_2:[0,1] \to [0,1]$  be two continuous fuzzy negations, $U_1, U_2:[0,1]^2 \to [0,1]$ two disjunctive uninorms with neutral elements $e_1,e_2 \in ]0,1[,$ respectively, and let
 $I(x,y)=U_1(N_1(x),y)=U_2(N_2(x),y).$
 If $U_1$ has continuous underlying functions
and $N_1\neq N_2$ then also $U_2$ has continuous underlying functions
and there exist idempotent points $a,d\in [0,1]$
such that the restriction of $U_1$
 (resp. of $U_2$)
to $]a,d[^2$ is a linear transformation of some representable uninorm restricted to $]0,1[^2$. Moreover, $U_1$ coincides with $U_2$ on $([0,a]\cup [d,1])^2.$
\end{proposition}
\begin{proof}
From Proposition \ref{proposition:2} we know that $N_1\neq N_2$ implies $U_1\neq U_2$ and $e_1\neq e_2.$  Moreover,  
$$N_2(x) = U_2(N_2(x),e_2) = I(x,e_2)=U_1(N_1(x),e_2),$$
and from Proposition \ref{prop:non_uniq} we know that $U_1(\cdot,e_2)$ is a continuous function with range $[0,1].$
Therefore, there exists $x_e\in [0,1]$ such that $U_1(x_e,e_2) = e_1.$
Observe that since $U_1(e_1,e_2) =e_2$ then $e_1>e_2$ ($e_1<e_2$) implies $x_e>e_1$ ($x_e<e_1$).
Due to \cite[Proposition 8]{Mesiarova2018} we know that  there exist idempotent points $a,d\in [0,1]$ of $U_1,$ $a\leq d,$ such that
$x_e,e_2,e_1\in ]a,d[$ and $U_1$ is continuous on
$]a,d[\times[0,1] \cup [0,1]\times ]a,d[.$ Due to the monotonicity of $U_1$ and Lemma 2 and Proposition 8 from \cite{Mesiarova2018} we know that $U_1(a,x) = a$ and $U_1(d,x)=d$ for all $x\in ]a,d[.$ Therefore, the restriction of $U_1$ to $[a,d]^2$ is a uninorm on $[a,d]$ which is non-continuous only in points $(a,d)$ and $(d,a).$ Proposition 6 from \cite{Ruiz} then shows that  $U_1$ restricted to $]a,d[^2$ is a linear transformation of some representable uninorm restricted to $]0,1[^2.$ 
Lemma 10 from \cite{Mesiarova2017OS} then implies \begin{equation}\label{eq1}U_1(x,y)=\min(x,y)
\text{ for all } (x,y)\in ([0,a]\times [a,d[) \cup ([a,d[\times [0,a]),\end{equation} 
and \begin{equation}\label{eq2}U_1(x,y)=\max(x,y)
\text{ for all } (x,y)\in ([d,1]\times ]a,d]) \cup (]a,d]\times [d,1]).\end{equation}
Thus, all cuts $U_1(\cdot,x)$ for $x\in ]a,d[$ are continuous with range $[0,1].$ However, for all $x\in [0,a]\cup [d,1]$ the cut $U_1(\cdot,x)$
is either non-continuous, or its range is not $[0,1].$
Indeed, in the opposite case we can show similarly as above that $U_1(x,y)=e_1$
for some $y\in [0,1]$ and Lemma 4 from \cite{Mesiarova2018} shows that $x\notin \{a,d\}$. However, then there exist idempotent points $a_2,d_2\in [0,1]$ of $U,$ $a_2\leq d_2$ such that $x,y,e_1\in ]a_2,d_2[,$
the restriction of $U_1$ to $]a_2,d_2[^2$ is a linear transformation of some representable uninorm and either $a\in ]a_2,d_2[,$ or $d\in ]a_2,d_2[,$
which is a contradiction, since each representable uninorm has only trivial idempotent points $e,0,1.$ 

Now, consider any $x\in ]a,d[.$ Since range of $U_1(\cdot,x)$ on $]a,d[$ is $]a,d[$ and $e_2\in ]a,d[,$  there exists $y\in ]a,d[$ such that $U_1(y,x)=e_2.$ 
Moreover, there exists $y_0 \in [0,1]$ such that $N_1(y_0)=y.$
Then $$U_2(N_2(y_0),x)=U_1(N_1(y_0),x) = U_1(y,x)=e_2.$$ Thus for all $x\in ]a,d[$ there exists $z\in [0,1]$ such that $U_2(z,x)=e_2.$ Moreover, for each $v\in [0,1]$ 
$$v=U_2(v,e_2) =U_2(v,U_2(z,x))=U_2(U_2(v,z),x),$$ i.e., the range of $U_2(\cdot,x)$  is $[0,1]$ and due to the monotonicity of $U_2$ the cut $U_2(\cdot,x)$ is continuous for all $x\in ]a,d[.$ Thus, $U_2$ is continuous on
$]a,d[\times[0,1] \cup [0,1]\times ]a,d[.$ 

\noindent Since $N_1$ is continuous, there exists   $x\in [0,1]$ such that $N_1(x)=a$ and  then 
$$N_2(x)=U_2(N_2(x),e_2)=U_1(N_1(x),e_2) = U_1(a,e_2)=a$$
and similarly $N_1(y)=d$ implies $N_2(y)=d.$ Therefore,
$U_2(a,a)=U_1(a,a)=a$ and $U_2(d,d)=U_1(d,d)=d.$ Thus, similarly as above,
$U_2$ restricted to $]a,d[^2$ is a linear transformation of some representable uninorm restricted to $]0,1[^2.$ 
\newline
For all $x\in [0,1]$ such that $N_1(x)\leq a$  we get $$N_2(x)=U_1(N_1(x),e_2)=N_1(x)$$
due to Eq. \eqref{eq1} and for all $x\in [0,1]$ such that $N_1(x)\geq d$ we get $$N_2(x)=U_1(N_1(x),e_2)=N_1(x)$$
due to Eq. \eqref{eq2}. Then, for each  $x\in [0,a]\cup [d,1]$  there exists $x^* \in [0,1]$ such that $N_2(x^*) = N_1(x^*)=x$.  and we have $$U_1(x,y) = U_1(N_1(x^*),y) = I(x^*,y) = U_2(N_2(x^*),y) = U_2(x,y),$$ 
for all  $y \in [0,1]$. Therefore, for any $x\in [0,a]\cup [d,1]$
we have $U_1(x,y)=U_2(x,y)$ for all $y\in [0,1].$ This shows that $U_1$ and $U_2$ can differ only on $]a,d[^2$
and since $U_2$ restricted to $]a,d[^2$ is a linear transformation of some representable uninorm restricted to $]0,1[^2$ we see that $U_2$ has continuous underlying functions.
\end{proof}

\begin{remark}
\begin{itemize}
\item[(i)]
From  \cite[Proposition 11]{Mesiarova2017OS} we know that each uninorm with continuous underlying functions can be expressed as an ordinal sum of semigroups related to continuous Archimedean t-norms, t-conorms, representable uninorms and trivial semigroups. Proposition \ref{proposition:3} shows that if $I(x,y)=U_1(N_1(x),y)=U_2(N_2(x),y)$ holds for all $x,y\in[0,1]$ for some continuous fuzzy negations $N_1,N_2$ with $N_1\neq N_2$ and disjunctive uninorms with continuous underlying functions $U_1,U_2$ then the corresponding ordinal sum decomposition of $U_1$ and $U_2$ differ only in one summand, which corresponds to the top element in the related linear order and in both cases it is a representable semigroup defined on $]a,d[$.
\item[(ii)] Let $U_1:[0,1]^2\to [0,1]$ be a disjunctive uninorm with continuous underlying functions and neutral element $e_1\in ]0,1[$ and let $N_1:[0,1]\to [0,1]$ be any continuous fuzzy negation. If the top summand in the ordinal sum decomposition of $U_1$ is a representable semigroup defined on $]a,d[$ then from the proof of Proposition \ref{proposition:2} and from  Proposition \ref{proposition:3} it follows that
for any $s\in ]a,d[\setminus \{e\}$ there exists a disjunctive uninorm  $U_2:[0,1]^2\to [0,1]$ with continuous underlying functions and neutral element $s$ such that 
$U_1(N_1(x),y)=U_2(N_2(x),y)$ holds for all $x,y\in[0,1],$ where $N_2:[0,1]\to [0,1]$ given by $N_2(x)=U_1(N_1(x),s)$ for all $x\in [0,1]$ is a continuous fuzzy negation.
\end{itemize}
\end{remark}

\begin{example} 
Let $a=\frac{1}{4},$ $d=\frac{3}{4},$ let
 $U^{'}_1:[a,d]^2\to [a,d]$ 
be a linear transformation of the uninorm $U_1$  from Example \ref{example:1} and let $U^{*}_1:]a,d[^2\to ]a,d[$  be the restriction of $U^{'}_1$  to $]a,d[.$
 If we define a function $U_3:[0,1]^2\to [0,1]$ by
$$U_3(x,y)=\begin{cases}
\frac{1}{4} + \frac{1}{2} \frac{(x-\frac{1}{4})(y-\frac{1}{4})}{(x-\frac{1}{4})(y-\frac{1}{4}) + (\frac{3}{4}-x)(\frac{3}{4}-y)} &\text{if $x,y\in ]\frac{1}{4},\frac{3}{4}[,$} \\
\max(x,y) &\text{if $\max(x,y)\geq \frac{3}{4},$} \\
\min(x,y) &\text{otherwise,}
\end{cases}$$
then $U_3$ is a disjunctive uninorm with continuous underlying functions with neutral element $e_3=\frac{1}{2}$, which can be expressed as ordinal sum of semigroups $G_1=(]a,d[,U_1^*),$  $G_2=([0,a],\min)$ and $G_3=([d,1],\max)$ with respect to linear order $\preceq$ given by $3\prec 2\prec 1$. Note that semigroups $([0,a],\min)$ and $([d,1],\max)$ can be further decomposed into ordinal sum of trivial semigroups. Let $N_1:[0,1]\to [0,1]$ be the continuous fuzzy negation given by $N_1(x) = 1-x$ for all $x\in [0,1].$
Then $(U,N)$-implication $I:[0,1]^2\to [0,1]$ with $I(x,y)=U_3(N_1(x),y)$  is  for all $x,y\in [0,1]$ given by
$$I(x,y)=\begin{cases}
\frac{1}{4} + \frac{(\frac{3}{4}-x)(y-\frac{1}{4})}{2(\frac{3}{4} - x)(y-\frac{1}{4}) +2(x-\frac{1}{4})(\frac{3}{4}-y) } &\text{if $x,y\in ]\frac{1}{4},\frac{3}{4}[,$} \\
\max(1-x,y) &\text{if $\max(1-x,y)\geq \frac{3}{4},$} \\
\min(1-x,y)  &\text{otherwise.}
\end{cases}$$
For any $s\in ]\frac{1}{4},\frac{3}{4}[$ we have
$I(x,s)=\frac{(\frac{3}{4}-x)(s-\frac{1}{4})}{2(\frac{3}{4} - x)(s-\frac{1}{4}) +2(x-\frac{1}{4})(\frac{3}{4}-s) } +\frac{1}{4}$ for all $x\in ]\frac{1}{4},\frac{3}{4}[$ and  $I(x,s)=1-x$ for all $x\in [0,\frac{1}{4}]\cup [\frac{3}{4},1].$
For example, for $s=\frac{3}{8}$ we have
$I(x,s)=\frac{3}{16x}$ for all $x\in ]\frac{1}{4},\frac{3}{4}[.$
Thus, $N_2:[0,1]\to [0,1]$ given by $ N_2(x) = I(x,\frac{3}{8})$ is a continuous fuzzy negation and $N_2(x)=N_2^{-1}(x).$ This shows us that the corresponding uninorm $U_4:[0,1]^2\to [0,1],$ such that $U_3(N_1(x),y)=U_4(N_2(x),y)$ is given by $I(N_2(x),y),$ i.e., 
$$U_4(x,y) = \begin{cases}
\frac{1}{4}+\frac{3(x-\frac{1}{4})(y-\frac{1}{4})}{6(x-\frac{1}{4})(y-\frac{1}{4}) +2(\frac{3}{4}-x)(\frac{3}{4}-y) }&\text{if $x,y\in  ]\frac{1}{4},\frac{3}{4}[,$} \\
\max(x,y) &\text{if $\max(x,y)\geq \frac{3}{4}$}, \\
\min(x,y) &\text{otherwise.} 
\end{cases}$$
It is not difficult to check that $U_4$ on $]\frac{1}{4},\frac{3}{4}[^2$ is just the linear transformation of the restriction of $U_2$ from Example \ref{example:1} to $]0,1[,$ which we denote by $U_2^*.$ Thus, $U_4$ can be expressed as ordinal sum of semigroups $H_1=(]\frac{1}{4},\frac{3}{4}[,U_2^*),$  $H_2=([0,\frac{1}{4}],\min)$ and $H_3=([\frac{3}{4},1],\max)$ with respect to linear order $\preceq$ given by $3\prec 2\prec 1$.
\end{example}

\subsection{Uninorms with non-continuous underlying functions}

In this section, we study uninorms with non-continuous underlying functions which do not ensure uniqueness of representation of $(U,N)$-implications. According to Proposition \ref{proposition:2} we may search for uninorms that have continuous, increasing horizontal cuts despite of not having continuous underlying functions. We have already seen in Example \ref{example:unique_rep} that this is not always the case, but in the following we show an example of two disjunctive uninorms, which do not have continuous underlying functions, however, $U_1(N_1(x),y)=U_2(N_2(x),y)$ for all $x,y\in [0,1]$ is satisfied for some continuous fuzzy negations $N_1$ and $N_2.$

\begin{example}
\label{ex.nonc}
    Let us consider $e_1 = \frac{1}{2}$ and $f : [0,1] \to [0,1]$ the function defined by $f(x)=x^2$. Since $f$ is a continuous, strictly increasing function with $f(0)=0$ and $f(1)=1$, we can define the following functions
    $$f^{(n)}(x) = \overbrace{f \circ \dots \circ f}^{n}  (x) = x^{2^n}, \quad f^{(-n)}(x) = \overbrace{f^{-1} \circ \dots \circ f^{-1}}^{n}  (x) = x^{\frac{1}{2^n}} = x^{2^{-n}},$$
    where $ n \in \mathbb{Z}^{+}$. We adopt the convention that $f^{(0)} = \text{id}$. In this case, the following points $u^{(n)} = f^{(n)}(e_1) = 2^{-2^n}$ for all $n \in \mathbb{Z}$ form a partition of $]0,1[$, i.e,
    $$
    ]0,1[ = \bigcup_{n \in \mathbb{Z}} ] u^{(n+1)}, u^{(n)}] = \bigcup_{n \in \mathbb{Z}} ] 2^{-2^{n+1}}, 2^{-2^n}].
    $$
    Having said this, we define the following function
    \begin{equation}\label{eqU1R}U_1(x,y)=\begin{cases}
1 &\text{if $\max(x,y)=1,$}\\
0 &\text{if $\min(x,y)=0$ and $\max(x,y)<1,$}\\
f^{(n+m)}(\min(f^{(-n)}(x),f^{(-m)}(y)))  &\text{if $x\in ]u^{(n+1)},u^{(n)}],$ $y\in ]u^{(m+1)},u^{(m)}].$}
\end{cases}\end{equation}
Notice that, if $x \in ]u^{(n+1)},u^{(n)}]$ then $f^{(-n)}(x) \in ]u,e_1]$. Then, it is clear that $U_1(x,y) \in ] u^{(n+m+1)},u^{(n+m)}]$ whenever $x\in ]u^{(n+1)},u^{(n)}],$ $y\in ]u^{(m+1)},u^{(m)}]$ with $n,m \in \mathbb{Z}$. Moreover,
\begin{eqnarray}\label{eq:1}
U_1(u^{(n)},y) &=& f^{(n+m)}(\min(f^{(-n)}(u^{(n)}),f^{(-m)}(y))) = f^{(n+m)}(\min(e_1,f^{(-m)}(y))) \\ \nonumber
&=& f^{(n+m)}(f^{(-m)}(y)) = f^{(n)}(y),
\end{eqnarray}
for all $y \in ]0,1[$ with $y \in ]u^{(m+1)},u^{(m)}]$ and $n \in \mathbb{Z}$. In particular, if $y = u^{(m)}$ then $U_1(u^{(n)},u^{(m)}) = f^{(n)}(u^{(m)}) = u^{(n+m)}$. Now, we will show that $U_1$ is a disjunctive uninorm. The commutativity of $U_1$ is obvious from Eq. (\ref{eqU1R}). If $y = e_1$ then $x=1$ implies $U_1(x,y)=1$, $x=0$ implies $U_1(x,y)=0$, and $x \in ]u^{(n+1)},u^{(n)}]$ implies $U_1(x,y)=f^{(0)}(x)=x$ by Eq. (\ref{eq:1}) with $n=0$ since $e_1 \in ]u,e_1]$. Thus, from the commutativity of $U_1$ we deduce that $e_1$ is the neutral element of $U_1$. For the associativity consider $x,y,z \in [0,1]$. If $\max(x,y,z)=1$ or $\min(x,y,z)=0$ the associativity is clear. Otherwise, let us consider $x \in ]u^{(n+1)},u^{(n)}]$, $y \in ]u^{(m+1)},u^{(m)}]$ and $z \in ]u^{(k+1)},u^{(k)}]$. Here
\begin{eqnarray*}
U_1(x,U_1(y,z)) &=& U_1(x, f^{(m+k)}(\min(f^{(-m)}(y),f^{(-k)}(z)))) \\
&=& f^{(n+m+k)}(\min(f^{(-n)}(x), f^{-(m+k)}\circ f^{(m+k)}(\min(f^{(-m)}(y),f^{(-k)}(z))))), \\
&=& f^{(n+m+k)}(\min(f^{(-n)}(x),f^{(-m)}(y),f^{(-k)}(z)),
\end{eqnarray*}
since $\min(f^{(-m)}(x),f^{(-k)}(x)) \in ]u,e_1]$ and $f^{(m+k)}(\min(f^{(-m)}(y),f^{(-k)}(z)))) \in ]u^{(m+k+1)},u^{(m+k)}]$. Analogously, we obtain
$$
U_1(U_1(x,y),z) = f^{(n+m+k)}(\min(f^{(-n)}(x),f^{(-m)}(y),f^{(-k)}(z)),
$$
and the associativity holds. Now, we have to show the monotonicity. Due to the commutativity we will show it only for the second coordinate. Consider $x,y,z \in [0,1]$ with $y< z$, then we distinguish between three cases:
\begin{itemize}
    \item If $x=1$ then $U_1(x,y)=1= U_1(x,z)$.
    \item If $x=0$ then $z<1$ implies $y<1$ and $U_1(x,y)=0=U_1(x,z)$, while $z=1$ implies $y<1$ and $U_1(x,y)=0<1=U_1(x,z)$.
    \item If $x \in ]u^{(n+1)},u^{(n)}]$ for some $n \in \mathbb{Z}$ and $y=0$ then $U_1(x,y)=0 \leq U_1(x,z)$. If $z=1$ then $U_1(x,y) \leq 1 = U_1(x,z)$. Otherwise, $y>0$ implies $z>0$ and $z<1$ implies $y<1$, i.e., $y \in ]u^{(m+1)},u^{(m)}]$ and $z \in ]u^{(k+1)},u^{(k)}]$ for some $m,k \in \mathbb{Z}$. Here, $y<z$ implies $m \geq k$ and we have to consider two more cases:
    \begin{itemize}
        \item If $m=k$ then $f^{(-m)}(y) < f^{(-m)}(z)$ and we get
        $$U_1(x,y) = f^{(n+m)}(\min(f^{(-n)}(x),f^{(-m)}(y))) \leq f^{(n+m)}(\min(f^{(-n)}(x),f^{(-k)}(z))) = U_1(x,z).$$
        \item If $m > k$ then $U_1(x,y) \in ]u^{(n+m+1)},u^{(n+m)}]$ and $U_1(x,z) \in ]u^{(n+k+1)},u^{(n+k)}]$ so we have
        $$
        U_1(x,y) \leq u^{(n+m)} \leq u^{(n+k+1)} < U_1(x,z).
        $$
    \end{itemize}
\end{itemize}
Finally, since $U_1(1,0)=1$ we know that $U_1$ is a disjunctive uninorm. Note that according to Eq. (\ref{eq:1}), all the cuts $U_1(u^{(n)},\cdot)$ for $n \in \mathbb{Z}$ correspond to $f^{(n)}$, so they are continuous, strictly increasing functions with $f^{(n)}(0)=0$ and $f^{(n)}(1) = 1$. However, $U_1$ does not have continuous underlying functions since it is non-continuous on the lower border of each rectangle
 $[u^{(n+1)},u^{(n)}]\times [u^{(m+1)},u^{(m)}].$
 Indeed, if $n,m \in \mathbb{Z}$ and $y \in ] u^{(m+1)},u^{(m)}]$ we get 
 $$U_1(u^{(n)},y) = f^{(n)}(y),$$
 and 
 \begin{eqnarray*}
 \lim_{s \to (u^{(n)})^{+}} U_1(s,y) &=& \lim_{s \to (u^{(n)})^{+}} f^{(n+m-1)}(\min(f^{(-n+1)}(s),f^{(-m)}(y))), \\
&=& f^{(n+m-1)}(\min(f^{(-n+1)}(u^{(n)}),f^{(-m)}(y))) \\
&=& f^{(n+m-1)}(\min(u,f^{(-m)}(y))) \\
&=& f^{(n+m-1)}(u) \\
&=& f^{(n)}(u^{(m)}),
 \end{eqnarray*}
 where the equality holds only if $y = u^{(m)}$. Next, we define the function $N_1 : [0,1] \to [0,1]$ by
$$N_1(x) = 2^{-2^{-\log_2(-\log_2(x))}}.$$
Then $N_1$ is a continuous, strictly decreasing fuzzy negation such that 
$$N_1(u^{(n)})=N_1(2^{-2^n})=2^{-2^{-n}}=u^{(-n)},$$
for all $n\in \mathbb{Z}.$ Now, we define the function $N_2:[0,1]\to [0,1]$ by
$$N_2(x)=U_1(N_1(x),u)=(N_1(x))^2.$$
Then, $N_2$ is also a continuous, strictly decreasing fuzzy negation and
$$ U_1(N_2(x),u^{(-1)})=U_1(U_1(N_1(x),u),u^{(-1)})=
U_1(N_1(x),U_1(u,u^{(-1)}))=U_1(N_1(x),e_1)=N_1(x).$$
Further, we define function $U_2:[0,1]^2\to [0,1]$ by $$U_2(x,y) = U_1(u^{(-1)},U_1(x,y)).$$
It is easy to check that $U_2$ is a disjunctive uninorm  with neutral element $u$ since all its properties follow from the corresponding properties of $U_1$ and the strict monotonicity and continuity of $f^{(-1)}.$
Then, for all $x,y\in [0,1]$ we get  
$$U_2(N_2(x),y) = U_1(u^{(-1)},U_1(N_2(x),y))=
U_1(U_1(u^{(-1)},N_2(x)),y)=U_1(N_1(x),y).$$
\end{example}

In the following, we will generalize the situation in Example \ref{ex.nonc} for an arbitrary uninorm $U$ with non-continuous underlying functions. We first point out that, if $U\in \mathcal{U}_x^{hc}$ then  $U$ is completely characterized in the square $]a,d[$ determined by the the limits of the sequence of the powers of $x$ by the horizontal cut, and the restriction of $U$ to $[x,e]$.

\begin{remark}
\label{remNC}
If 
$U\in \mathcal{U}_x^{hc}$
then
there exists $y\in ]0,1[$ such that $U(x,y)=e.$
Note that due to the monotonicity of $U$ and $\mathrm{Ran}(U(x,\cdot)) = [0,1]$ we have $U(x,0)=0$ and $U(x,1)=1.$
Since $U(x,e)=x$, due to the monotonicity of $U$, $x<e$ implies $y>e$ and $x>e$ implies $y<e.$
Then, for all $v\in [0,1]$ we have
$$v=U(e,v)=U(U(y,x),v)=U(y,U(x,v)),$$
i.e, also the cut $U(y,\cdot)$ has range $[0,1]$ and due to the monotonicity it is continuous.
If $U(x,s)=U(x,t)$ for some $s,t\in [0,1]$ then
$$s=U(e,s)=U(y,U(x,s))=U(y,U(x,t))=U(U(y,x),t)=U(e,t)=t$$
and thus the cut $U(x,\cdot)$ is strictly increasing and similarly we can show that the  cut $U(y,\cdot)$ is strictly increasing. 
Moreover,
$$U(U(x,x),U(y,y))=U(x,U(U(x,y),y))=U(x,U(e,y))=U(x,y)=e,$$
i.e., if we denote $x^{(0)}_U=e$ and $x_U^{(n)}=U(x,x_U^{(n-1)})$ for all $n\in \mathbb{N}$ and $y^{(0)}_U=e$ and $y_U^{(n)}=U(y,y_U^{(n-1)})$ for all $n\in \mathbb{N}$ we get similarly as above that $U(x_U^{(n)},y_U^{(n)})=e$ for all $n\in \mathbb{N}.$
Then, cuts  $U(x_U^{(n)},\cdot)$ and  $U(y_U^{(n)},\cdot)$
are continuous, strictly increasing, with range $[0,1]$ for all $n\in \mathbb{N}.$
Moreover, since $U(x_U^{(n)},y_U^{(n)})=e$  we will denote 
 $x_U^{(-n)}=y_U^{(n)}$ for all $n\in \mathbb{N}.$ 
 
 For simplicity we will further denote $f^{(n)}(s)=U(x_U^{(n)},s)$ and $f^{(-n)}(s)=U(y_U^{(n)},s)$ for $n\in \mathbb{N}$ and
$f^{(0)}(s)=s.$ Note that
$$
f^{(n)} \circ f^{(-n)}(s) = U(x_U^{(n)},U(y_U^{(n)},s)) = U(e,s)=s,
$$ and 
$$
f^{(n)} \circ f^{(m)}(s) = U(x_U^{(n)},U(x_U^{(m)},s)) = U(x_U^{(n+m)},s)= f^{(n+m)}(s),
$$
for all $s \in [0,1]$ and $n,m\in \mathbb{Z}$. Further, each $f^{(n)}$ for $n\in \mathbb{Z}$ is completely determined by $U(x,\cdot).$
 The strict monotonicity of each cut $U(x_U^{(n)},\cdot)$ for $n\in \mathbb{Z}$ implies that
 if
 $s\in ]x_U^{(n+1)},x_U^{(n)}]$ for some $n\in \mathbb{Z}$ then $U(x_U^{(m)},s)\in ]x_U^{(n+m+1)},x_U^{(n+m)}]$ for all $m\in \mathbb{Z}.$ 
 In addition, $s_1=f^{(-n)}(s) \in ]x,e]$  is the unique value such that $U(s_1,x_U^{(n)}) =f^{(n)}(s_1) =s.$

Further we will consider $x<y$ since otherwise it is enough just to redefine our points. Consider a function $F:[x,e]^2\to [0,1]$ given by $$F(s,t)=U(s,t) \text{   for all $s,t\in [x,e].$}$$
Then $F$ is non-decreasing, $F(s,t)=F(t,s)$ for all $s,t\in [x,e]$ and   
$F(s,e)=s$ for all $s\in [x,e].$
Moreover, for all $s,t\in ]a,d[,$
where $a = \lim\limits_{n\to +\infty} x_U^{(n)}$
and $d = \lim\limits_{n\to -\infty} y_U^{(n)},$
there exists exactly one $n\in \mathbb{Z}$ and
exactly one $m\in \mathbb{Z}$ 
such that $s\in ]x_U^{(n+1)},x_U^{(n)}]$
and  $t\in ]x_U^{(m+1)},x_U^{(m)}].$
Then $f^{(-n)}(s),f^{(-m)}(t) \in ]x,e]$ and we get
\begin{eqnarray}\label{eq:2}
    U(s,t) &=& U(f^{(n)} \circ f^{(-n)}(s),f^{(m)} \circ f^{(-m)}(t)) = U(x_U^{(n+m)}, U(f^{(-n)}(s),f^{(-m)}(t))) \\ \nonumber
    &=& f^{(n+m)}(U(f^{(-n)}(s),f^{(-m)}(t)))) = f^{(n+m)}(F(f^{(-n)}(s),f^{(-m)}(t)))).
\end{eqnarray}
Thus, $U$ on $]a,d[^2$ is completely determined by $U(x,\cdot)$ and the function $F$,  i.e., the restriction of $U$ to $[x,e]$.

Due to the monotonicity of $U$ and since $U(x,\cdot)$ is strictly increasing, we know that $F(s,t)\in ]x_U^{(2)},e]$
for all $s,t\in ]x,e].$ Since $U$ is associative, for any $s,t,u\in ]x,e]$ if $F(s,t)\in ]x_U^{(n+1)},x_U^{(n)}]$
and $F(t,u)\in ]x_U^{(m+1)},x_U^{(m)}]$ for some $n,m \in \{0,1\}$ then 
 by Eq. (\ref{eq:2}) we obtain the following
\begin{eqnarray}
\label{eqas}
f^{(m)}(F(s,f^{(-m)}(F(t,u))))  &=&  U(s,F(t,u)) =  U(s,U(t,u)) =  U(U(s,t),u) \\ \nonumber
&=& U(F(s,t),u) 
= f^{(n)}(F(f^{(-n)}(F(s,t)),u)).
\end{eqnarray}
Summarizing, $F$ is non-decreasing, commutative, has neutral element $e,$ satisfies Eq. \eqref{eqas} for all 
$s,t,u\in ]x,e]$ and $F(x,s)=f(s)$ for all $s\in [x,e].$

\end{remark}

From the previous discussion we see that if we want to construct a  uninorm $U\in \mathcal{U}_x^{hc}$ then we have to solve two problems. How to define $U$ on $]a,d[^2$ and how to define $U$ on $[0,1]^2\setminus ]a,d[^2.$
For the first problem we have the following answer.
We can either start from $x<e$ or from $x>e,$ however, since these two approaches are similar, we will take $x<e.$
Moreover, for each   continuous, strictly increasing function $f:[0,1]\to [0,1]$  with $f(0)=0,$  $f(1)=1$ we denote $f^{(0)}(x)=x,$  $f^{(n)}(x)=f(f^{(n-1)}(x))$ and $f^{(-n)}(x)=f^{-1}(f^{(-n+1)}(x))$ for all $x\in [0,1],$ $n\in \mathbb{N},$ where $f^{-1}$ is the inverse function of $f.$

\begin{proposition}\label{prop:F}
    Let $f:[0,1]\to [0,1]$ be a continuous, strictly increasing function with $f(0)=0,$  $f(1)=1$ and $f(y)=e$ for some $y,e\in ]0,1[$ with $y>e$ and denote $x=f(e).$ Further denote $\lim\limits_{n\to +\infty}f^{(n)}(x) = a$
and $\lim\limits_{n\to +\infty}f^{(-n)}(x) = d.$ Let $F:[x,e]^2\to [f(x),e]$ be a commutative, non-decreasing function such that  $F(s,e)=s$ and $F(s,x)=f(s)$ for all $s\in [x,e],$
and $F$ satisfies Eq. \eqref{eqas}
for all $s,t,u\in ]x,e]$ such that
$F(s,t)\in ]f^{(n+1)}(e),f^{(n)}(e)]$
and $F(t,u)\in ]f^{(m+1)}(e),f^{(m)}(e)]$ for some $n,m \in \{0,1\}$. Then, function $U:]a,d[^2 \to ]a,d[$ given by
$$
U(s,t) = f^{(n+m)}(F(f^{(-n)}(s),f^{(-m)}(s))), \quad  \text{if $s\in ]f^{(n+1)}(e),f^{(n)}(e)],$ $t\in ]f^{(m+1)}(e),f^{(m)}(e)],$}
$$
for some $n,m\in \mathbb{Z}$ is a uninorm on $]a,d[,$ i.e., a commutative, associative, non-decreasing function with neutral element, defined on $]a,d[.$
\end{proposition}

\begin{proof}
    Since $f$ is a continuous, strictly increasing function we know that for each $s \in ]a,d[$ there exists exactly one $n \in \mathbb{Z}$ such that $s \in ]f^{(n+1)}(e),f^{(n)}(e)]$. Thus, $U$ is well defined. Note that $U(x,s) = f(s) = F(x,s)$ for all $s \in ]x,e]$ and $F(x,x)=f(x)=f^{(2)}(e)$. Since $F$ is commutative, also $U$ is commutative. Moreover, for $s \in ] f^{(n+1)}(e),f^{(n)}(e)]$ we have
    $$
    U(s,e) = f^{(n)}(F(f^{(-n)}(s),e)) = f^{(n)} \circ f^{(-n)}(s) =s,
    $$
    i.e., $e$ is the neutral element of $U$. For the associativity, let us assume $s,t,u \in ]a,d[$ and let
$s\in ]f^{(p+1)}(e),f^{(p)}(e)],$ $t\in ]f^{(q+1)}(e),f^{(q)}(e)]$ and $u\in ]f^{(r+1)}(e),f^{(r)}(e)],$ for some $p,q,r \in \mathbb{Z}$.  For simplicity, we denote $s_1 = f^{(-p)}(s)$, $t_1 = f^{(-q)}(t)$ and $u_1 = f^{(-r)}(u)$. 
Then $U(s,t) =f^{(p+q)}(F(s_1,t_1)) $ and $U(t,u) = f^{(q+r)}(F(t_1,u_1)).$
Since $s_1,t_1,u_1 \in ]x,e]$, then $F(s_1,t_1),F(t_1,u_1) \in ]f(x),e]$ and we need to distinguish between several cases:
\begin{itemize}
\item If $F(s_1,t_1) \in ]x,e]$
then 
$$ U(s,U(t,u))=U(s,f^{(q+r)}(F(t_1,u_1)))=
f^{(p+q+r)}(F(s_1,F(t_1,u_1)),$$
and we separate between two more cases:
\begin{itemize}
    \item  If $F(t_1,u_1) \in ]x,e]$ then
    \begin{equation}\label{eq:ass_1}
   U(U(s,t),u)= U(f^{(p+q)}(F(s_1,t_1)),u)=f^{(p+q+r)}(F(F(s_1,t_1),u_1)),
    \end{equation}
    and since Eq. \eqref{eqas} implies $F(s_1,F(t_1,u_1)) = F(F(s_1,t_1),u_1),$ the associativity holds.
    \item $F(t_1,u_1) \in ]f(x),x]$ then
        \begin{equation}\label{eq:ass_2} U(U(s,t),u)=
        U(f^{(p+q)}(F(s_1,t_1)),u)=f^{(p+q+r+1)}(F(f^{-1}(F(s_1,t_1)),u_1)),
        \end{equation}
    and since Eq. \eqref{eqas} implies $F(s_1,F(t_1,u_1)) = f(F(f^{-1}(F(s_1,t_1)),u_1)),$ the associativity holds.
\end{itemize}
\item If $F(t_1,u_1) \in ]f(x),x]$ then
$$U(s,f^{(q+r)}(F(t_1,u_1)))=
f^{(p+q+r+1)}(F(s_1,f^{-1}(F(t_1,u_1))),$$
and we separate between two cases again:
\begin{itemize}
    \item If $F(t_1,u_1) \in ]x,e]$ then $U(s,U(t,u))$ is given by Eq. (\ref{eq:ass_1}) and since Eq. (\ref{eqas}) implies\\ $f(F(s_1,f^{-1}(F(t_1,u_1))) = F(F(s_1,t_1),u_1)$, the associativity holds.
    \item If $F(t_1,u_1) \in ]f(x),x]$ then $U(s,U(t,u))$ is given by Eq. (\ref{eq:ass_2}) and since Eq. (\ref{eqas}) implies\\ $f(F(s_1,f^{-1}(F(t_1,u_1))) = f(F(f^{-1}(F(s_1,t_1)),u_1))$, the associativity holds.
\end{itemize}
\end{itemize}
Finally, since $U$ is commutative is enough to show the monotonicity only for the second coordinate. Let us assume $s,t,u \in ]a,d[$ with $t <u$,
$s\in ]f^{(p+1)}(e),f^{(p)}(e)],$ $t\in ]f^{(q+1)}(e),f^{(q)}(e)]$ and $u\in ]f^{(r+1)}(e),f^{(r)}(e)],$ for some $p,q,r \in \mathbb{Z}$ with $s_1 = f^{(-p)}(s)$, $t_1 = f^{(-q)}(t)$ and $u_1 = f^{(-r)}(u)$. In this case, either $r=q$ and $t_1 < u_1$, or $q >r$. If $r=q$ then the monotonicity of $F$ implies
$$
U(s,t) = f^{(p+q)}(F(s_1,t_1)) \leq f^{(p+q)}(F(s_1,u_1)) =U(s,u).
$$
Otherwise, if $q>r$ we consider different cases depending on the value of $F(s_1,t_1)$ and $F(s_1,u_1)$:
\begin{itemize}
    \item If $F(s_1,t_1),F(s_1,u_1) \in ]x,e]$ then
    \begin{eqnarray*}
        U(s,t) &=& f^{(p+q)}(F(s_1,t_1)) \leq f^{(p+q)}(e) \leq f^{(p+r+1)}(e) = f^{(p+r)}(x) \\
        &<& f^{(p+q)}(F(s_1,u_1)) = U(s,u).
    \end{eqnarray*}
    \item $F(s_1,t_1) \in ]f(x),x]$ and $F(s_1,u_1) \in ]x,e]$ then
    \begin{eqnarray*}
        U(s,t) &=& f^{(p+q)}(F(s_1,t_1)) \leq f^{(p+q)}(x) = f^{(p+q+1)}(e) < f^{(p+r+1)}(e) = f^{(p+r)}(x) \\
        & < & f^{(p+r)}(F(s_1,u_1)) = U(s,u).
    \end{eqnarray*}
    \item $F(s_1,t_1),F(s_1,u_1) \in ]f(x),x]$ implies
    \begin{eqnarray*}
        U(s,t) &=& f^{(p+q)}(F(s_1,t_1)) \leq f^{(p+q)}(x) = f^{(p+q+1)}(e) \leq f^{(p+q+2)}(e) \\
        &=&f^{(p+r)}(f(x)) < f^{(p+r)}(F(s_1,u_1)) = U(s,u),
    \end{eqnarray*}
     \item $F(s_1,t_1)\in ]x,e]$, $F(s_1,u_1) \in ]f(x),x]$ implies
    \begin{eqnarray*}
        U(s,t) &=& f^{(p+q)}(F(s_1,t_1)) \leq f^{(p+q)}(s_1) \leq f^{(p+r+1)}(s_1) = f^{(p+r)}(f(s_1)) \\
        &\leq& f^{(p+r)}(F(s_1,u_1)) = U(s,u),
    \end{eqnarray*}   
since $F$ is non-decreasing and $F(s_1,t_1)\leq F(s_1,e)=s_1,$ $f(s_1)=F(s_1,x)\leq F(s_1,u_1).$
\end{itemize}
Summarizing, $U$ is a uninorm on $]a,d[^2$. 
\end{proof}

Note that if $F(s,t)\in ]x,e]$ holds for all $s,t\in ]x,e]$
then  Eq. \eqref{eqas} is equivalent to the associativity of $F.$ Proposition \ref{prop:F} and Remark \ref{remNC} completely characterize and describe the construction of $U\in \mathcal{U}_x^{hc}$ on 
 $]a,d[^2.$ However, the structure of $U$ on $[0,1]^2\setminus ]a,d[^2$ is determined only on cuts $U(x^{(n)}_U,\cdot)$ 
for $n\in \mathbb{Z}.$

In the following we will denote for $U\in \mathcal{U}_x^{hc}$ with $x<e$ $$a_x = \lim\limits_{n\to +\infty} x_U^{(n)} \quad
\text{ and } \quad d_x = \lim\limits_{n\to +\infty} x_U^{(-n)}.$$

\begin{lemma}
\label{lemNW}
Let  $U\in \mathcal{U}_x^{hc},$ $x<e$ and $f:[0,1]\longrightarrow [0,1]$ be given by $f(s)=U(x,s)$ for all $s \in [0,1]$. Then the following statements hold:
\begin{enumerate}[label=(\roman*)]
\item $f(s)\neq s$ for all $s\in ]a_x,d_x[,$
\item $f(s)=s$ for some $s\in [0,1]$ implies $U(s,y)\notin ]a_x,d_x[$ for all $y\in [0,1],$ 
\item $f(a_x)=a_x$ and $f(d_x)=d_x,$
\item $U(a_x,d_x)\in \{a_x,d_x\}.$
\end{enumerate}
\end{lemma}

\begin{proof}
\begin{enumerate}[label=(\roman*)]
\item
By Remark \ref{remNC} for each $s\in ]a_x,d_x[$ there exists exactly one $n\in \mathbb{N}$ such that $s\in ]x^{(n+1)}_U,x^{(n)}_U]$
and then $f(s)\in ]x^{(n+2)}_U,x^{(n+1)}_U].$
Since $]x^{(n+1)}_U,x^{(n)}_U]\cap ]x^{(n+2)}_U,x^{(n+1)}_U] = \emptyset$ we get $f(s)\neq s$ for all $s\in ]a_x,d_x[.$
\item If $f(s)=s $ for some $s\in [0,1]$ and $U(s,y)\in ]a_x,d_x[$ for some $y\in [0,1]$ then  previous item implies $$U(s,y)\neq f(U(s,y))=U(f(s),y)=U(s,y),$$ which is  a contradiction.
\item For all $s\in ]a_x,d_x[$ we have $f(s)\in ]a_x,d_x[$ and for all $s<a_x$ we have $f(s) = U(x,s)\leq \min(x,s)<a_x$ since $x<e$ and $s<a_x<e.$ Then the continuity of $f$ implies $f(a_x)=a_x.$ Similarly we can show that   $f(d_x)=d_x.$
\item By the monotonicity of $U$ we get $a_x=U(a_x,e)\leq U(a_x,d_x)\leq U(e,d_x)=d_x,$ i.e., $U(a_x,d_x)\in [a_x,d_x].$
Since $f(a_x)=a_x$ we get $U(a_x,d_x)\notin ]a_x,d_x[,$ i.e., $U(a_x,d_x)\in \{a_x,d_x\}.$
\end{enumerate}
\end{proof}

In one particular case we are able to characterize $U\in \mathcal{U}_x^{hc}$ with $x<e$  on whole $[0,1]^2.$
 If $U(x,t)=t$ for all $t\in [0,a_x]\cup [d_x,1]$  we obtain the following result.

\begin{proposition}
Let $U\in \mathcal{U}_x^{hc}$ with $x<e.$
If $U(x,t)=t$ for all $t\in [0,a_x]\cup [d_x,1]$ then $U$ is an ordinal sum of semigroups defined on $]a_x,d_x[$ and on $[0,a_x]\cup [d_x,1].$
\end{proposition}

\begin{proof} Consider function $f$ and $F$ from Remark \ref{remNC}.
If $U(x,t)=t$ for all $t\in [0,a_x]\cup [d_x,1]$ then 
$U(x_U^{(2)},t)=U(U(x,x),t) = U(x,U(x,t))=U(x,t)=t$ and similarly we can show that  $U(x_U^{(n)},t)=t$ for all $t\in [0,a_x]\cup [d_x,1]$ holds for all $n\in \mathbb{Z}.$
Due to the monotonicity we obtain $U(s,t)=t$ for all $s\in ]a_x,d_x[$ and $t\in [0,a_x]\cup [d_x,1].$
Moreover, we know that $]a_x,d_x[$ is closed under $U$ due to Remark \ref{remNC}, Eq. \eqref{eq:2} and the definition of $a_x$ and $d_x.$
What remains is to show that $[0,a_x]\cup [d_x,1]$ is closed under $U$.
On the contrary, let us assume that  $U(s,t)\in ]a_x,d_x[$
for some $s,t\in [0,a_x]\cup [d_x,1].$
Then for all $u\in ]a_x,d_x[$ we have $$U(U(s,t),u)=U(s,U(t,u))=U(s,t),$$ i.e., $U(s,t)$ is the annihilator of $U$ on $]a_x,d_x[.$
However, then $f(U(s,t))=U(x,U(s,t))=U(s,t)\in ]a_x,d_x[,$
which is a contradiction with Lemma \ref{lemNW}. Thus $[0,a_x]\cup [d_x,1]$ is closed under $U$
and 
$U$ is an ordinal sum of semigroups defined on $]a_x,d_x[$ and on $[0,a_x]\cup [d_x,1].$
\end{proof}

From the previous result we see that for construction of a  uninorm $U\in \mathcal{U}_x^{hc}$  via ordinal sum construction we need only some strictly increasing continuous function  $f:[0,1]\to[0,1]$
with range $[0,1],$ a non-decreasing commutative function
$F:[x,e]^2\to [x,e]$
such that $e$ is the neutral element of $F,$ $F(x,s)=f(s)$ for $s\in [x,e]$ and it fulfills Eq. \eqref{eqas} jointly with any generalized composite uninorm defined on $[0,a_x]\cup [d_x,1]$ (see \cite{Mesiarova2016OSG}).

\begin{example}
\label{exOS}
Let $a,d\in ]0,1[,$ $a<d.$
Consider uninorm $U_1$ from Example \ref{ex.nonc}
and let $U_1^{'}$ be the linear transformation of $U_1$ to 
$[a,d]$ and let $U_1^*$ be the restriction of $U_1^{'}$ to $]a,d[^2.$ Since $]0,1[^2$ is closed under $U_1$ we know that $U_1^*$ is well defined. 
On $[0,a]$ consider the standard product $x\cdot y$
and on $[d,1]$ consider the dual operation to standard product, i.e., $x\diamond y = x+y-x\cdot y.$
Then, semigroup $G_1=([0,a],\cdot)$ corresponds to a continuous t-subnorm, semigroup 
$G_2=([d,1],\diamond)$ to a continuous  t-superconorm and $G_3=(]a,d[,U_1^*])$  to a uninorm. 
Ordinal sum of $G_i$ for $i\in \{1,2,3\}$ with respect to linear order $\preceq$ given by $2\prec 1 \prec 3$ is a disjunctive uninorm $U$, which has infinitely many continuous horizontal cuts  with range $[0,1].$ Obviously, $U$ does not have continuous underlying functions.
\end{example}

From  previous example we see that
a uninorm $U\in \mathcal{U}_x^{hc}$  can have a point of discontinuity in every interval  $]x,y[$ such that $e\in ]x,y[.$ 
Note that even if we would take in the previous example instead of $G_3$ any representable semigroup defined on $]a,d[$ still the resulting uninorm won't have continuous underlying functions.

However, in the following example we will see that a uninorm $U\in \mathcal{U}_x^{hc}$ with $x<e,$ which does not have continuous underlying functions, need not be an ordinal sum of a semigroup acting on $]a_x,d_x[$ and a semigroup acting on $[0,a_x]\cup [d_x,1].$

\begin{example} 
Let $U_1^*$ be defined as in Example \ref{exOS} for 
$a=\frac{1}{4}$ and $d=\frac{3}{4}.$ Then $U_1^*$ is a uinorm on $]a,d[$ with neutral element
$e=\frac{1}{2}$ and a continuous horizontal cut $f:]\frac{1}{4},\frac{3}{4}[ \longrightarrow ]\frac{1}{4},\frac{3}{4}[ $ in   $u=\frac{3}{8}$ given by  
$$f(s) =U_1^*(u,s)= 2\cdot (s-\frac{1}{4})^2 +\frac{1}{4},$$
for $s\in  ]\frac{1}{4},\frac{3}{4}[.$ 
We will extend $f$ to $[0,1]$ as follows.
$$f(s) = \begin{cases}
4\cdot s^2 &\text{if $s\leq \frac{1}{4},$}\\
2\cdot (s-\frac{1}{4})^2 +\frac{1}{4} &\text{if $s\in  ]\frac{1}{4},\frac{3}{4}[,$}\\
4\cdot(s-\frac{3}{4})^2+\frac{3}{4}  &\text{if $s\geq \frac{3}{4}.$}
\end{cases}$$
Then $f$ is continuous, strictly increasing, $f(0)=0,$ 
$f(1)=1$, $f(\frac{1}{4})=\frac{1}{4}$ and $f(\frac{3}{4})=\frac{3}{4}.$
We denote $f^{(0)}(s)=s,$  $f^{(n)}(s)=f(f^{(n-1)}(s))$ and $f^{(-n)}(s)=f^{-1}(f^{(-n+1)}(s))$ for all $s\in [0,1],$ $n\in \mathbb{N},$ where $f^{-1}$ is the inverse function to $f.$
Here, for all $n\in \mathbb{Z}$ and for $s\in ]\frac{1}{4},\frac{3}{4}[$ we have
$$f^{(n)}(s) = \frac{1}{4} + \frac{1}{2}\cdot \left( 2\cdot(s-\frac{1}{4})\right)^{2^n},$$
for $s\leq \frac{1}{4}$ we have
$$f^{(n)}(s) = \frac{1}{4} \cdot \left(4\cdot s\right)^{2^n},$$
and for $s\geq \frac{3}{4}$
we have $$f^{(n)}(s) = \frac{3}{4} + \frac{1}{4}\cdot \left( 4\cdot(s-\frac{3}{4})\right)^{2^n}.$$
We define a function $U:[0,1]^2\to [0,1]$ as follows.
\begin{equation} \label{eqUf} U(x,y)=\begin{cases} 
0 &\text{if $x,y\in [0,a],$} \\
1 &\text{if $x,y\in [d,1[,$} \\
d &\text{if $(x,y)\in [0,a]\times [d,1[\cup [d,1[\times [0,a],$} \\
1  &\text{if $\max(x,y)=1,$} \\
f^{(n)}(x) &\text{if $x\in [0,a]\cup [d,1], y\in ]f^{(n+1)}(e),f^{(n)}(e)]$ for $n\in \mathbb{Z},$} \\
f^{(n)}(y) &\text{if $y\in [0,a]\cup [d,1], x\in ]f^{(n+1)}(e),f^{(n)}(e)]$ for $n\in \mathbb{Z},$} \\
U_1^*(x,y) &\text{if $x,y\in ]a,d[.$}
\end{cases}\end{equation} 
We will show that $U$ is a disjunctive uninorm. 
From Eq. \eqref{eqUf} we immediately see that $U$ is commutative and disjunctive. If $x\in ]a,d[$
then $U(x,e)=x$ since $U_1^*$ is a linear transformation of $U_1$ restricted to $]0,1[.$ If 
$x\in [0,a]\cup [0,d]$ then $U(x,e)=f^{(0)}(x)=x$
since $e\in ]f(e),f^{(0)}(e)].$ Thus the commutativity of $U$ implies that $e$ is the neutral element of $U.$
For associativity consider $x,y,z\in [0,1].$ In order to have a small number of cases we will compare in each case with points from different domains values $U(x,U(y,z)),U(U(x,y),z)
$ and  $U(y,U(x,z)),$ which will then due to the commutativity cover all possible symmetric cases. Then we have the following.
\begin{itemize}
\item If $1\in \{x,y,z\}$ then   $U(x,U(y,z))=1=U(U(x,y),z)
.$
\item If $x,y,z\in [0,a]$ then $U(x,U(y,z))=0=U(U(x,y),z).$
\item If $x,y,z\in [d,1[$ then $U(x,U(y,z))=1=U(U(x,y),z)
.$
\item If $x,y,z\in ]a,d[$ then $U(x,U(y,z))=U(U(x,y),z)$ follows from the associativity of $U_1$ from Example \ref{ex.nonc}.
\item If $x\in [0,a]$ and $y,z\in [d,1[$
then $U(x,U(y,z)) = U(x,1)=1,$ $U(U(x,y),z)=U(d,z)=1$
and $U(y,U(x,z))=U(y,d)=1.$
\item If $x\in [d,1[$ and $y,z\in [0,a]$ then 
 $U(x,U(y,z)) = U(x,0)=d,$ $U(U(x,y),z)=U(d,z)=d$
and $U(y,U(x,z))=U(y,d)=d.$
\item If $x\in [0,a]$ and $y,z\in ]a,d[$
with $y\in ]f^{(n+1)}(e),f^{(n)}(e)],$ $z\in ]f^{(m+1)}(e),f^{(m)}(e)],$ $y_1=f^{(-n)}(y)\in ]u,e]$ and $z_1=f^{(-m)}(z)\in ]u,e],$ 
then $U(x,U(y,z)) = U(x,f^{(m+n)}(\min(y_1,z_1)))=f^{(m+n)}(x),$ $U(U(x,y),z)=U(f^{(n)}(x),z)=f^{(m+n)}(x)$
and $U(y,U(x,z))=U(y,f^{(m)}(x))=f^{(m+n)}(x).$
\item If $x\in ]a,d[$ and $y,z\in [0,a]$
with  $x\in ]f^{(n+1)}(e),f^{(n)}(e)],$ $x_1=f^{(-n)}(x)\in ]u,e]$ 
then 
 $U(x,U(y,z)) = U(x,0)=f^{(n)}(0)=0,$ $U(U(x,y),z)=U(f^{(n)}(y),z)=0$ since $f^{(n)}(y)\leq a$
and similarly $U(y,U(x,z))=U(y,f^{(n)}(z))=0.$
\item If $x\in [d,1[$ and $y,z\in ]a,d[$  
with $y\in ]f^{(n+1)}(e),f^{(n)}(e)],$ $z\in ]f^{(m+1)}(e),f^{(m)}(e)],$ $y_1=f^{(-n)}(y)\in ]u,e]$ and $z_1=f^{(-m)}(z)\in ]u,e],$ 
then 
$U(x,f^{(m+n)}(\min(y_1,z_1))) = f^{(n+m)}(x),$ $U(U(x,y),z)=U(f^{(n)}(x),z)=f^{(n+m)}(x)$ 
and $U(y,U(x,z))=U(y,f^{(m)}(x))=f^{(n+m)}(x).$

\item If $x\in ]a,d[$ and $y,z\in [d,1[$ with  $x\in ]f^{(n+1)}(e),f^{(n)}(e)],$ $x_1=f^{(-n)}(x)\in ]u,e]$  then  $U(x,U(y,z)) = U(x,1)=1,$ $U(U(x,y),z)=U(f^{(n)}(y),z)=1$ since $f^{(n)}(y)\geq d$
and similarly we have $U(y,U(x,z))=U(y,f^{(n)}(z))=1.$

\item If $x\in [0,a],$ $y\in ]a,d[$ and $z\in [d,1[$ with $y\in ]f^{(n+1)}(e),f^{(n)}(e)],$ $y_1=f^{(-n)}(y)\in ]u,e]$ then $U(x,U(y,z)) = U(x,f^{(n)}(z))=d,$ since $1>f^{(n)}(z)\geq d,$ $U(U(x,y),z)=U(f^{(n)}(x),z)=d$
since $f^{(n)}(x)\leq a$
and $U(y,U(x,z))=U(y,d)=f^{(n)}(d)=d.$ 
\end{itemize}

Due to the commutativity it is enough to show the monotonicity only for the second coordinate. Consider $x,y,z\in [0,1]$ with $y<z.$ Then $y<1$
and if  $z=1$ then $U(x,y)\leq 1=U(x,z).$ Further we will suppose that $z<1.$
We have the following cases.
\begin{itemize}
\item If $x=1$ then $U(x,y)=1=U(x,z).$
\item If $x\in [d,1[$ then $y\in [d,1[$ implies  $z\in [d,1[$ and  $U(x,y)=1=U(x,z).$
\item  If $x\in [d,1[$  and $y\in ]f^{(n+1)}(e),f^{(n)}(e)]$ for some $n\in \mathbb{Z}$ then $z\in [d,1[$
gives us $U(x,y)=f^{(n)}(x)<1=U(x,z).$
If $z\in ]f^{(m+1)}(e),f^{(m)}(e)]$ for some $m\in \mathbb{Z}$
then $n\geq m$ and $U(x,y)=f^{(n)}(x)\leq f^{(m)}(x)=U(x,z).$
\item If $x\in [d,1[$  and $y\in [0,a]$
then $z\in [0,a]$ implies $U(x,y)=d=U(x,z),$
$z\in [d,1[$ implies $U(x,y)=d<1=U(x,z)$
and $z\in ]f^{(m+1)}(e),f^{(m)}(e)]$ for some $m\in \mathbb{Z}$ implies $U(x,y)=d\leq f^{(m)}(x)=U(x,z).$
\item If $x\in ]f^{(n+1)}(e),f^{(n)}(e)]$ for some $n\in \mathbb{Z}$ then $y\in [d,1[$ implies  $z\in [d,1[$ and  $U(x,y)=f^{(n)}(y)\leq f^{(n)}(z)=U(x,z)$ since $f$ is strictly increasing.
\item If $x\in ]f^{(n+1)}(e),f^{(n)}(e)]$ for some $n\in \mathbb{Z}$ and $y\in ]f^{(m+1)(e)},f^{(m)}(e)]$ for some $m\in \mathbb{Z}$ then $z\in ]a,d[$ implies  $U(x,y)\leq U(x,z)$ due to the monotonicity of $U_1^*$ and $z\in [d,1[$ implies $U(x,y)\leq U(x,z)$ since $U(x,z)=f^{(n)}(z)\geq d$ and $U(x,y)\in ]a,d[.$
\item If $x\in ]f^{(n+1)}(e),f^{(n)}(e)]$ for some $n\in \mathbb{Z}$ and $y\in [0,a]$
then $z\in [0,a]$ implies $U(x,y)=f^{(n)}(y)\leq f^{(n)}(z)=U(x,z)$ since $f$ is strictly increasing.
If $z\in ]f^{(m+1)}(e),f^{(m)}(e)]$ for some $m\in \mathbb{Z}$
then $U(x,y)\leq U(x,z)$ since $U(x,y)=f^{(n)}(y)\leq a$ and $U(x,z)\in ]a,d[.$ If $z\in [d,1[$
then $U(x,y)=f^{(n)}(y)\leq a<d\leq f^{(n)}(z)=U(x,z).$
\item If $x\in [0,a]$  then $y\in [d,1[$ implies  $z\in [d,1[$ and $U(x,y)=d=U(x,z).$
\item If $x\in [0,a]$ and $y\in ]f^{(n+1)}(e),f^{(n)}(e)]$ for some $n\in \mathbb{Z}$ then  $z\in ]f^{(m+1)}(e),f^{(m)}(e)]$ for some $m\in \mathbb{Z}$ implies
$n\geq m$ and 
$U(x,y)=f^{(n)}(x)\leq f^{(m)}(x)=U(x,z).$
If $z\in [d,1[$ then $U(x,y)=f^{(n)}(x)\leq a< d=U(x,z).$
\item If $x,y\in [0,a]$ then 
$U(x,y)=0\leq U(x,z).$
\end{itemize}
Summarizing, $U$ is a disjunctive uninorm with infinitely many continuous horizontal cuts with range $[0,1],$ however, $U$ cannot be expressed as an ordinal sum of semigroups acting on $]a,d[$ and on $[0,a]\cup [d,1].$
\end{example}




\section{Conclusions}\label{section:conclusions}

In this paper, we have deeply studied the non-uniqueness of representation of $(U,N)$-implications. The starting point of our work has been revisiting the characterization of these operators, published in \cite{Backzynski2009B}, and proving that in general uniqueness of representation is not guaranteed as previously indicated by the authors. Apart from providing examples of the non-uniqueness claim, we have proved that under the assumption of a continuous negation any representation of a $(U,N)$-implication is linked to the existence of a continuous horizontal cut which is decreasing and its range is $[0,1]$. This result has motivated the study of characterizing uninorms which do not provide a unique representation when used for constructing $(U,N)$-implications with a continuous fuzzy negation $N$. We have separated this study for uninorms with and without continuous underlying functions. In the first case, we have been able to characterize all uninorms with a suitable horizontal cut.
For the second case, we have only been able to characterize the structure of the uninorm in a certain subsquare of $[0,1]^2$. However, we have pointed out a specific case in which we are able to determine the values of the uninorm in the rest of the domain to obtain an operator with the desired behavior. Since it is well-known that the structure of uninorms without continuous underlying functions is much more complex and it is not characterized, to solve the problem in general has been marked as a possible future work. For all the results, we have provided examples of the explicit construction of the corresponding uninorms. Our results not only deepen the understanding of the structure of $(U,N)$-implications and related operators but also offer novel insights into the structural properties of uninorms.

\section*{Acknowledgements}
Raquel Fernandez-Peralta was funded by the EU NextGenerationEU through the Recovery and Resilience Plan for Slovakia under the project No. 09I03-03-V04- 00557. Andrea Mesiarová-Zemánková was supported by grant VEGA 1/0036/23 and partially by the project  AIMet4AI No. CZ.02.1.01/0.0/0.0/17\_049/0008414.

\bibliographystyle{splncs04.bst}
\bibliography{main}

\end{document}